\documentclass{article}

\usepackage{amsfonts}
\usepackage{amsmath}
\usepackage{mathtools}
\usepackage{multirow}
\usepackage{epstopdf}
\usepackage{algorithm}
\usepackage{algorithmic}
\usepackage{bm}
\usepackage{microtype}
\usepackage{amssymb}
\usepackage{enumerate}
\usepackage{marvosym}
\usepackage{color}
\usepackage{marvosym}
\usepackage{subcaption}
\usepackage{hyperref}

\newtheorem{theorem}{Theorem}

\newtheorem{remark}{Remark}
\newtheorem{lemma}{Lemma}
\newtheorem{proof}{Proof}
\newtheorem{corollary}{Corollary}

\newtheorem{assumption}{Assumption}

\newcommand{\x}{\mathbf{x}}
\newcommand{\y}{\mathbf{y}}

\newcommand{\g}{\mathbf{g}}

\newcommand{\m}{\mathbf{m}}
\newcommand{\n}{\mathbf{n}}

\renewcommand{\v}{\mathbf{v}}

\newcommand{\E}{\mathbb{E}}
\newcommand{\bO}{{\cal O}}
\newcommand{\R}{\mathbb{R}}

\newcommand{\F}{\mathcal{F}}

\newcommand{\bsigma}{\boldsymbol{\sigma}}

\newcommand{\<}{\left\langle}
\renewcommand{\>}{\right\rangle}

\usepackage[accepted]{icml2020}

\icmltitlerunning{Convergence Rate of AdamW Measured by $\ell_1$ Norm}

\begin{document}

\onecolumn
\icmltitle{On the $\bO(\frac{\sqrt{d}}{K^{1/4}})$ Convergence Rate of AdamW \\Measured by $\ell_1$ Norm}

\begin{icmlauthorlist}
\icmlauthor{Huan Li}{to}
\icmlauthor{Yiming Dong}{goo}
\icmlauthor{Zhouchen Lin}{goo}
\end{icmlauthorlist}

\icmlaffiliation{to}{Institute of Robotics and Automatic Information Systems, College of Artificial Intelligence, Nankai University, Tianjin, China.\\}
\icmlaffiliation{goo}{National Key Lab of General AI, School of Intelligence Science and Technology, Peking University, Beijing, China.\\}
\icmlcorrespondingauthor{Huan Li and Zhouchen Lin}{lihuanss@nankai.edu.cn, zlin@pku.edu.cn}





\vskip 0.3in



\printAffiliationsAndNotice{}  

\begin{abstract}
  As the default optimizer for training large language models, AdamW has achieved remarkable success in deep learning. However, its convergence behavior is not well understood theoretically. This paper establishes the convergence rate $\frac{1}{K}\sum_{k=1}^K\E\left[\|\nabla f(\x^k)\|_1\right]\leq \bO(\frac{\sqrt{d}C}{K^{1/4}})$ for AdamW measured by the $\ell_1$ norm, where $K$ represents the iteration number, $d$ denotes the number of model parameters, and $C$ matches the constant in the optimal convergence rate of SGD. Theoretically, we have $\|\nabla f(\x)\|_2\ll \|\nabla f(\x)\|_1\leq \sqrt{d}\|\nabla f(\x)\|_2$ for any high-dimensional vector $\x$ and $\E\left[\|\nabla f(\x)\|_1\right]\geq\sqrt{\frac{2d}{\pi}}\E\left[\|\nabla f(\x)\|_2\right]$ when each element of $\nabla f(\x)$ is generated from a Gaussian distribution $\mathcal N(0,1)$. Empirically, our experimental results on real-world deep learning tasks reveal $\|\nabla f(\x)\|_1=\varTheta(\sqrt{d})\|\nabla f(\x)\|_2$. Both support that our convergence rate can be regarded as analogous to the optimal $\frac{1}{K}\sum_{k=1}^K\E\left[\|\nabla f(\x^k)\|_2\right]\leq \bO(\frac{C}{K^{1/4}})$ convergence rate of SGD in the ideal case. We also extend our result to NAdamW, an AdamW variant that employs a double-momentum mechanism, and demonstrate that it maintains the same convergence rate.  \footnotetext[1]{This manuscript expands upon our NeurIPS 2025 conference version \citep{Li-2025-nips} by incorporating the analysis of NAdamW and by relaxing the coordinate-wise bounded noise variance assumption to the more general bounded noise variance assumption.}
\end{abstract}

\section{Introduction}
AdamW, which modifies Adam by decoupling weight decay from gradient-based updates, has emerged as the dominant optimizer for training deep neural networks, particularly for large language models. AdamW represents the pinnacle of adaptive gradient algorithms, having developed through the progression of AdaGrad \citep{Duchi-2011-jmlr,McMahan-2010-colt}, RMSProp \citep{RMSProp-2012-hinton}, Adam \citep{adam-15-iclr}, and finally AdamW \citep{adanw-2019-iclr} itself. Although the literature on the convergence analysis of adaptive gradient algorithms is quite extensive, there has been little research on the convergence properties of AdamW.

Recently, \cite{xie-2024-adamw-icml} proved that if the iterates of AdamW converge to some $\x_{\infty}$, then $\x_{\infty}$ is a KKT point of the constrained problem
\begin{eqnarray}
\begin{aligned}\label{problem}
\min_{\x\in\R^d} f(\x),\quad s.t.\quad \|\x\|_{\infty}\leq\frac{1}{\lambda},
\end{aligned}
\end{eqnarray}
where $f(\x)=\E_{\zeta\sim\mathcal{P}}[f(\x;\zeta)]$, $\zeta$ is the sample drawn from the data distribution $\mathcal{P}$, and $\lambda$ is the weight-decay parameter. Moreover, $\x$ is a KKT point of problem (\ref{problem}) iff \citep{xie-2024-adamw-icml}
\begin{eqnarray}
\begin{aligned}\label{kkt-condtion}
\|\x\|_{\infty}\leq\frac{1}{\lambda}\quad\mbox{and}\quad\<\lambda\x,\nabla f(\x)\>+\|\nabla f(\x)\|_1=0. 
\end{aligned}
\end{eqnarray}
\cite{xie-2024-adamw-icml} characterized which solution AdamW converges to, if it indeed converges. The next fundamental question to address is whether and how fast AdamW converges. \cite{zhou-pami-2024} conducted preliminary exploration on this problem. However, their analysis requires the weight-decay parameter to decrease exponentially, making AdamW reduce to Adam finally. To the best of our knowledge, aside from \citep{zhou-pami-2024}, we have not found any other literature addressing the convergence issue of AdamW.
 
In practical deep learning training, we often initialize the network weights to small values and employ modest weight decay, for example, $\lambda=0.01$, which empirically confines the optimization trajectory within the $\ell_{\infty}$ norm constraint, as empirically demonstrated in Figure \ref{figure3}. That is, $\|\x\|_{\infty}\leq\frac{c}{\lambda}$ for some $c<1$, making $\<\lambda\x,\nabla f(\x)\>+\|\nabla f(\x)\|_1$ lower bounded by $(1-c)\|\nabla f(\x)\|_1$. This key property enables the use of $\|\nabla f(\x)\|_1$ as an effective yet significantly simpler convergence metric for AdamW in practical settings. Building on this observation, this paper focuses on the convergence rate of AdamW within the constraint in problem (\ref{problem}). 

\noindent\begin{minipage}{0.45\textwidth}
\begin{algorithm}[H]
    \caption{AdamW}
    \label{adamw}
    \begin{algorithmic}
       \STATE Hyperparameters: $\eta,\theta,\beta,\lambda,\varepsilon$
       \STATE Initialize $\x^1$, $\m^0=0$, $\v^0=0$
       \FOR{$k=1,2,\cdots,K$}
       \STATE $\g^k=\mbox{GradOracle}(\x^k)$
       \STATE $\m^k=\theta\m^{k-1}+(1-\theta)\g^k$
       \STATE $\v^k=\beta\v^{k-1}+(1-\beta)(\g^k)^{\odot2}$
       \STATE $\x^{k+1}=(1-\lambda\eta)\x^k-\frac{\eta}{\sqrt{\v^k+\varepsilon}}\odot \m^k$
       \ENDFOR
       \STATE
    \end{algorithmic}
\end{algorithm}
\end{minipage}
\begin{minipage}{0.55\textwidth}
\begin{algorithm}[H]
    \caption{NAdamW}
    \label{nadamw}
    \begin{algorithmic}
       \STATE Hyperparameters: $\eta,\theta,\beta,\tau,\lambda,\varepsilon$
       \STATE Initialize $\x^1$, $\m^0=0$, $\v^0=0$
       \FOR{$k=1,2,\cdots,K$}
       \STATE $\g^k=\mbox{GradOracle}(\x^k)$
       \STATE $\m^k=\theta\m^{k-1}+(1-\theta)\g^k$
       \STATE $\v^k=\beta\v^{k-1}+(1-\beta)(\g^k)^{\odot2}$
       \STATE $\widetilde \m^k=\tau\m^k+(1-\tau)\g^k$
       \STATE $\x^{k+1}=(1-\lambda\eta)\x^k-\frac{\eta}{\sqrt{\v^k+\varepsilon}}\odot \widetilde \m^k$
       \ENDFOR
    \end{algorithmic}
  \end{algorithm}
\end{minipage}

\subsection{Contribution}
This paper investigates the convergence properties of AdamW. Specifically, we prove the following convergence rate for AdamW
\begin{eqnarray}
\begin{aligned}\label{AdamW-rate}
\frac{1}{K}\sum_{k=1}^K\E\left[\|\nabla f(\x^k)\|_1\right]\leq \bO\left(\frac{\sqrt{d}}{K^{1/4}}\sqrt[4]{\sigma^2L(f(\x^1)-f^*)}+\sqrt{\frac{dL(f(\x^1)-f^*)}{K}}\right)
\end{aligned}
\end{eqnarray}
under appropriate parameter settings such that $\|\x^k\|_{\infty}<\frac{1}{\lambda}$ for all iterates, where $K$ is the total iteration number, $d$ is the model dimension, $\sigma$ is the gradient noise variance, $L$ is the Lipschitz smoothness constant, and $f^*$ is a lower bound of $f(\x)$. Recall the classical convergence rate of SGD \citep{Bottou-2018-siam}
\begin{eqnarray}
\begin{aligned}\label{SGD-rate}
\frac{1}{K}\sum_{k=1}^K\E\left[\|\nabla f(\x^k)\|_2\right]\leq \bO\left(\frac{\sqrt[4]{\sigma^2L(f(\x^1)-f^*)}}{K^{1/4}}\right),
\end{aligned}
\end{eqnarray}
which matches the lower bound of non-convex stochastic optimization \citep{Arjevani-2023-mp}. Comparing (\ref{AdamW-rate}) with (\ref{SGD-rate}), we see that our convergence rate (\ref{AdamW-rate}) also achieves the same lower bound with respect to $K$, $\sigma$, $L$, and $f(\x^1)-f^*$. The only coefficient whose tightness remains unclear is the dimension $d$. Theoretically, we have $\|\nabla f(\x)\|_2\ll \|\nabla f(\x)\|_1\leq \sqrt{d}\|\nabla f(\x)\|_2$ for any high-dimensional vector $\x$ and $\E\left[\|\nabla f(\x)\|_1\right]\geq\sqrt{\frac{2d}{\pi}}\E\left[\|\nabla f(\x)\|_2\right]$ when each element of $\nabla f(\x)$ is drawn from a Gaussian distribution $\mathcal N(0,1)$. Empirically, we have observed $\|\nabla f(\x)\|_1=\varTheta(\sqrt{d})\|\nabla f(\x)\|_2$ on real-world deep learning tasks, as shown in Figure \ref{figure2}. Thus, we may regard that our convergence rate (\ref{AdamW-rate}) is analogous to (\ref{SGD-rate}) of SGD in the ideal case.

As a special case, we also establish the same convergence rate (\ref{AdamW-rate}) for Adam under slightly more relaxed parameter settings than those for AdamW. To the best of our knowledge, this convergence rate only appears for RMSProp first proved in \citep{lihuan-rmsprop-2024}, and similar results for AdaGrad subsequently appeared in \citep{jiang-adagrad-2024,zhangtong-adagrad-2024} and RMSProp in \citep{xie-rmsprop-iclr-2024} under different assumptions. Notably, comparable convergence guarantees remain unproven for AdamW and Adam. 

Finally, we extend our theory to NAdamW \citep{Dozat-nadam-iclr-2016} and demonstrate that it maintains the same convergence rate. NAdamW is a variant of AdamW that incorporates a double-momentum mechanism and performs slightly better than AdamW verified in a recent benchmark paper \citep{kaiyue-compare-2025}.

\subsection{Notation and Assumptions}

Denote $\F_k=\sigma(\g^1,\g^2,\cdots,\g^k)$ as the sigma field of the stochastic gradients up to $k$, denote $\E_{\F_k}[\cdot]$ as the expectation with respect to $\F_k$ and $\E_k[\cdot|\F_{k-1}]$ the conditional expectation with respect to $\g^k$ conditioned on $\F_{k-1}$. Denote $\|\cdot\|_1$, $\|\cdot\|_2$, and $\|\cdot\|_{\infty}$ as the $\ell_1$, $\ell_2$, and $\ell_{\infty}$ norms for vectors, respectively. For simplicity, we also use $\|\cdot\|$ as the $\ell_2$ norm. Denote $\odot$ for the Hadamard product. 

We make the following assumptions throughout this paper:
\begin{assumption}
Smoothness: $\|\nabla f(\y)-\nabla f(\x)\|\leq L\|\y-\x\|, \forall \x,\y$.
\end{assumption}
\begin{assumption}
Unbiased estimator: $\E_k\left[\g^k|\F_{k-1}\right]=\nabla f(\x^k), \forall \x^k$.
\end{assumption}
\begin{assumption}
Bounded noise variance: $\E_k\left[\|\g^k-\nabla f(\x^k)\|^2|\F_{k-1}\right]\leq\sigma^2, \forall \x^k$.
\end{assumption}
In contrast, our conference paper \citep{Li-2025-nips} employed a slightly stronger coordinate-wise bounded noise variance condition $\E_k\left[|\g_i^k-\nabla_i f(\x^k)|^2|\F_{k-1}\right]\leq\sigma_i^2$ and used $\sigma_s^2=\sum_{i=1}^d \sigma_i^2$ in the convergence rate, where $\sigma_s^2$ may be larger than $\sigma^2$.


\section{Convergence Rate of AdamW}\label{sec2}

This section presents our convergence rate result for AdamW. Algorithm \ref{adamw} provides the complete AdamW implementation, where setting the weight-decay parameter $\lambda=0$ recovers the standard Adam. For analytical simplicity, we omit the bias correction term in our analysis.

Based on Assumptions 1-3, we provide the convergence rate of AdamW in the following theorem. Note that we do not assume the boundedness of the gradient $\nabla f(\x^k)$ or stochastic gradient $\g^k$.

\begin{theorem}\label{theorem1}
Suppose that Assumptions 1-3 hold. Define $\hat\sigma^2=\max\left\{\sigma^2,\frac{L(f(\x^1)-f^*)}{K\gamma^2}\right\}$ with any constant $\gamma\in(0,1]$. Let $1-\theta=\sqrt{\frac{L(f(\x^1)-f^*)}{K\hat\sigma^2}}$, $\theta\leq \beta\leq\sqrt{\theta}$\footnote{We gratefully thank the anonymous NeurIPS reviewer to derive this looser bound. Our original bound is $\theta\leq \beta\leq\frac{(1+\theta)^2}{4}$.}, $\eta=\sqrt{\frac{f(\x^1)-f^*}{4KdL}}$, $\varepsilon=\frac{\hat\sigma^2}{d}$, $\lambda\leq\frac{\sqrt{d}}{10K^{3/4}}\sqrt[4]{\frac{L^3}{\hat\sigma^2(f(\x^1)-f^*)}}$, and $\|\x^1\|_{\infty}\leq \frac{5}{8}\sqrt{\frac{K(f(\x^1)-f^*)}{dL}}$. Then for AdamW, we have $\|\x^k\|_{\infty}<\frac{1}{\lambda}$ for all $ k=1,2,\cdots,K$ and
\begin{eqnarray}
\begin{aligned}\notag
\frac{1}{K}\sum_{k=1}^K\E\left[\|\nabla f(\x^k)\|_1\right]\leq \frac{10\sqrt{d}}{K^{1/4}}\sqrt[4]{\hat\sigma^2L(f(\x^1)-f^*)} + 66\sqrt{\frac{dL(f(\x^1)-f^*)}{K}}.
\end{aligned}
\end{eqnarray}
In particular, when $\sigma^2\leq\frac{L(f(\x^1)-f^*)}{K\gamma^2}$, we have $1-\theta=\gamma$, $\theta\leq \beta\leq\sqrt{\theta}$, $\eta=\sqrt{\frac{f(\x^1)-f^*}{4KdL}}$, $\varepsilon=\frac{L(f(\x^1)-f^*)}{dK\gamma^2}$, $\lambda\leq\frac{1}{10}\sqrt{\frac{dL\gamma}{K(f(\x^1)-f^*)}}$, $\|\x^1\|_{\infty}\leq \frac{5}{8}\sqrt{\frac{K(f(\x^1)-f^*)}{dL}}$, $\|\x^k\|_{\infty}<\frac{1}{\lambda}$, and accordingly
\begin{eqnarray}
\begin{aligned}\notag
\frac{1}{K}\sum_{k=1}^K\E\left[\|\nabla f(\x^k)\|_1\right]\leq 76\sqrt{\frac{dL(f(\x^1)-f^*)}{K\gamma}}.
\end{aligned}
\end{eqnarray}

\end{theorem}

Theorem \ref{theorem1} demonstrates that AdamW minimizes the gradient norm directly while restricting $\|\x^k\|_{\infty}<\frac{1}{\lambda}$. By contrast, $\ell_2$-regularized Adam only minimizes $\|\nabla f(\x)+\lambda\x\|$, rather than $\|\nabla f(\x)\|$. We provide the proof of an enhanced variant of AdamW in Section \ref{sec:proof} and summarize the key proof ideas in Remarks \ref{remark1} and \ref{remark2}. As a special case, we also establish the same convergence rate for Adam in the following corollary under slightly relaxed parameter settings. The complete description of an enhanced variant is given in Section \ref{sec:corollary}.

\begin{corollary}\label{corollay1}
With the same assumptions and parameter settings of $1-\theta$, $\eta$, and $\varepsilon$ as in Theorem \ref{theorem1}, but only requiring $0\leq\beta\leq 1$ rather than both $\theta\leq \beta\leq\sqrt{\theta}$ and $\|\x^1\|_{\infty}\leq\frac{5}{8}\sqrt{\frac{K(f(\x^1)-f^*)}{dL}}$, we have the same convergence rate for Adam as established in Theorem \ref{theorem1}. 
\end{corollary}

\subsection{Optimality of Our Convergence Rate}\label{sec2-1} 

When comparing our convergence rate (\ref{AdamW-rate}) with the optimal rate (\ref{SGD-rate}) of SGD, which aligns with the lower bound in non-convex stochastic optimization, we observe that our rate is also optimal with respect to $K$, $\sigma$, $L$, and $f(\x^1)-f^*$. The only remaining uncertainty concerns the tightness of the dimension $d$. Theoretically, $\|\nabla f(\x)\|_2\ll \|\nabla f(\x)\|_1\leq \sqrt{d}\|\nabla f(\x)\|_2$ holds for any high-dimensional vector $\x$, and when each element of $\nabla f(\x)$ is drawn from a Gaussian distribution $\mathcal N(0,1)$, we have $\E\left[\|\nabla f(\x)\|_1\right]\geq\sqrt{\frac{2d}{\pi}}\E\left[\|\nabla f(\x)\|_2\right]$ from Lemma \ref{main-lemma}. Empirically, experiments on real deep neural networks training confirm $\|\nabla f(\x)\|_1=\varTheta(\sqrt{d})\|\nabla f(\x)\|_2$, as demonstrated in Figure \ref{figure2}. Thus, our convergence rate (\ref{AdamW-rate}) can be regarded as analogous to SGD's optimal rate (\ref{SGD-rate}) in the ideal case.

\begin{figure}[t]
   \includegraphics[width=\linewidth]{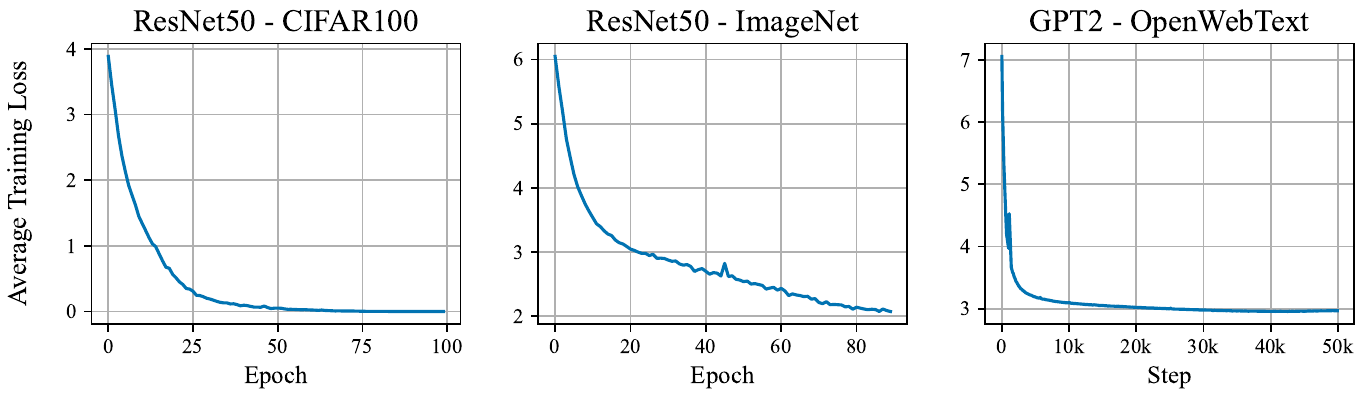}
   \caption{Illustration of average training loss $f(\x^k)$ for AdamW over epochs/steps, and at the initialization, $f(\x^1)\leq 8$.}
   \label{figure1}
\end{figure}

\begin{figure}[t]
  \includegraphics[width=\linewidth]{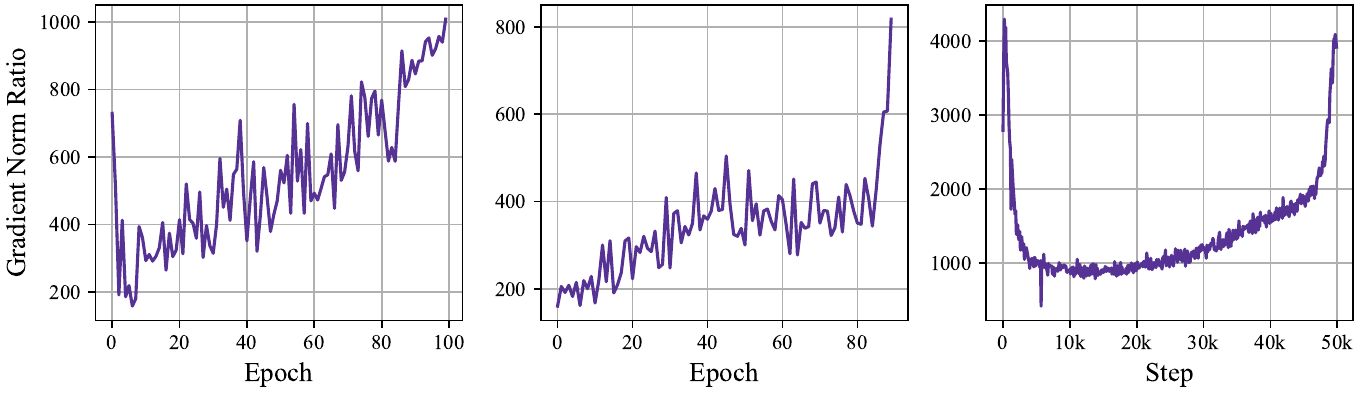}
  \caption{Illustration of $\|\nabla f(\x^k)\|_1=\varTheta(\sqrt{d})\|\nabla f(\x^k)\|_2$ for AdamW over epochs/steps. The gradient norm ratio shows $\frac{\|\nabla f(\x^k)\|_1}{\|\nabla f(\x^k)\|_2}$, and $\sqrt{d}= 4868$, $ 5060$, and $11136$, respectively.}
  \label{figure2}
\end{figure}

\begin{lemma}\label{main-lemma}
When each entry of $\x\in\mathbb R^d$ is generated from a Gaussian distribution with zero mean and unit variance, we have $\E\left[\|\x\|_1\right]\geq\sqrt{\frac{2d}{\pi}}\E\left[\|\x\|_2\right]$.
\end{lemma}
Recently, \cite{jiang-adagrad-2024} established a fundamental lower bound for SGD when measuring gradients by the $\ell_1$ norm, which is of order $\Omega (\sqrt{\frac{dL(f(\x^1)-f^*)}{K}}+\sqrt[4]{\frac{dL(f(\x^1)-f^*)(\sum_{i=1}^d\sigma_i)^2}{K}})$ under Assumptions 1-2 and the coordinate-wise bounded noise variance assumption. When $\sum_{i=1}^d\sigma_i\approx \sqrt{d\sum_{i=1}^d\sigma_i^2}\approx\sqrt{d}\sigma$, this lower bound aligns with our convergence rate in (\ref{AdamW-rate}). We further conjecture that this lower bound applies more broadly to general first-order stochastic optimization algorithms under $\ell_1$ norm gradient measurement. This would imply that our derived convergence rate is tight.

\begin{figure}[t]
  \includegraphics[width=\linewidth]{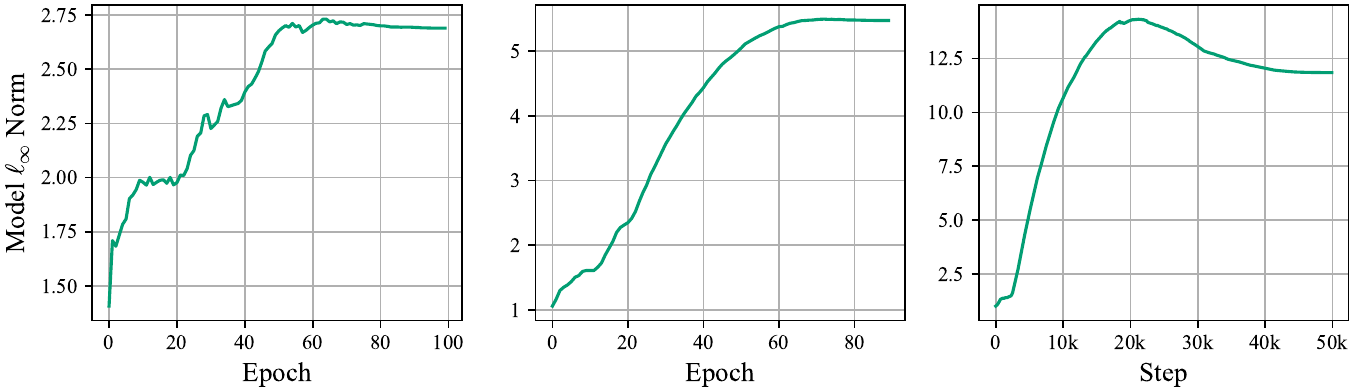}
  \caption{Illustration of $\|\x^k\|_{\infty}<\frac{1}{\lambda}$ for AdamW over epochs/steps. The model $\ell_{\infty}$ norm shows $\|\x^k\|_{\infty}$, and $\lambda=0.01$, $0.1$, and $0.05$, respectively.}
\label{figure3}
\end{figure}

\begin{figure}[t]
  \includegraphics[width=\linewidth]{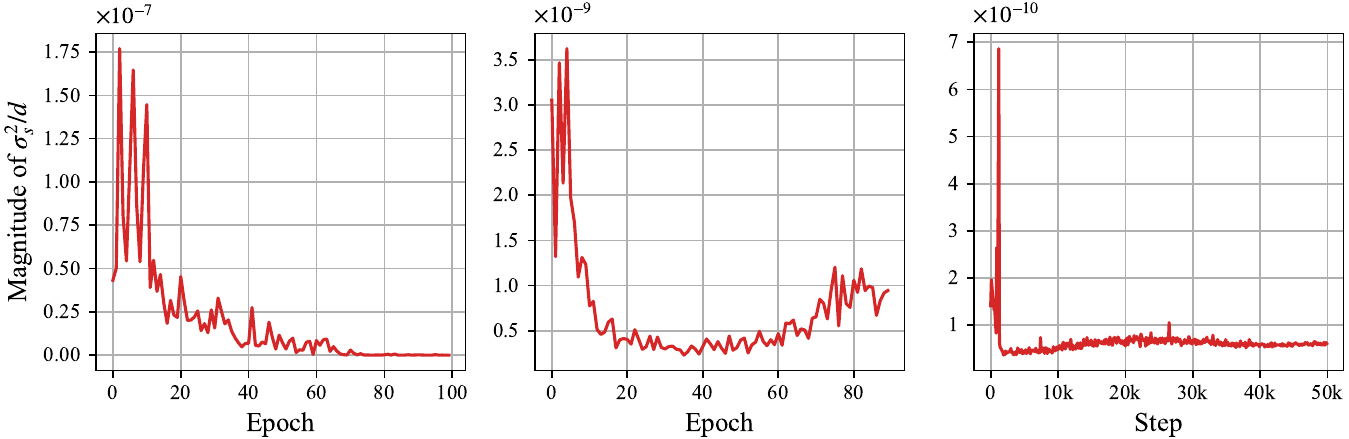}
  \caption{Illustration of small $\frac{\sigma_s^2}{d}$ over epochs/steps. The magnitude $\sigma_s^2$ is approximated by $\|\g^k-\nabla f(\x^k)\|^2$ for AdamW without taking expectation, and $d=2.37\times 10^7$, $2.56\times 10^7$, and $1.24\times 10^8$, respectively.}
  \label{figure4}
\end{figure}

\subsection{Separating the Convergence Rate by the Noise Variance} 

In Theorem \ref{theorem1}, we separate the convergence rate by the magnitude of $\sigma$. When $\sigma^2\geq\frac{L(f(\x^1)-f^*)}{K\gamma^2}$, both the convergence rates of AdamW and Adam are $\bO(\frac{\sqrt{d}}{K^{1/4}})$. When $\sigma^2$ becomes smaller than $\frac{L(f(\x^1)-f^*)}{K\gamma^2}$, the convergence rates improve to $\bO(\sqrt{\frac{d}{K}})$, matching that of gradient descent measured by the $\ell_1$ norm. 

\subsection{Reasonable Weight-Decay Parameter and Initialization Interval}

In Theorem \ref{theorem1}, we set the weight-decay parameter $\lambda$ smaller than $\frac{\sqrt{d}}{10K^{3/4}}\sqrt[4]{\frac{L^3}{\hat\sigma^2(f(\x^1)-f^*)}}$. In modern deep neural networks, the dimension $d$ is typically extremely large, for example, $d=1.75\times 10^{11}$ in GPT-3, making $\frac{\sqrt{d}}{K^{3/4}}$ almost certainly exceed $0.01$, which is the default setting of $\lambda$ in the official PyTorch implementation. For example, in the experiments of our paper, we train ResNet-50 on i) CIFAR-100 and ii) ImageNet dataset, and GPT-2 on iii) OpenWebText, and observe $(K,d)=(39100, 2.37\times 10^7)$, $(28080, 2.56\times 10^7)$, and $( 50000, 1.24\times 10^8)$, resulting in $\frac{\sqrt{d}}{K^{3/4}}\approx 1.75$, $2.33$, and $3.33$, respectively. Through the following example, we demonstrate that the upper bound of $\lambda$ is necessary for convergence. Specifically, we consider the following function, where $g(x)$ is the stochastic gradient:
\begin{eqnarray}
\begin{aligned}\notag
f(x)=\frac{(x-x^*)^2}{200},\mbox{ with }g(x)=\left\{
\begin{array}{ll}
x-x^*-1, & \mbox{with probability } p=0.1, \\
-\frac{1}{10}(x-x^*-\frac{10}{9}), & \mbox{with probability }1-p. \\
\end{array}
\right.
\end{aligned}
\end{eqnarray}
We set $K=10^{10}$, $\theta=1-\frac{1}{\sqrt{K}}$, $\beta=\sqrt{\theta}$, $\eta=\frac{1}{\sqrt{K}}$, $\varepsilon=10^{-10}$, $m^0=0$, $v^0=0$, and $x^1=x^*+1$ for AdamW, where $x^*=5$ is the minimizer of $f(x)$. We test $\lambda=\{10^{-1},10^{-2},10^{-3},10^{-4},10^{-5},0\}$ such that $x^*<\frac{1}{\lambda}$ and thus the KKT conditions (\ref{kkt-condtion}) reduce to $|\nabla f(x^*)|=0$ at the minimum solution. So we can use the gradient norm $|\nabla f(x)|$ to measure the convergence. From Figure \ref{figure5}, we see that AdamW fails to converge to $x^*$ when $\lambda=\{10^{-1},10^{-2},10^{-3}\}$, indicating that large values of $\lambda$ exceeding a certain threshold may cause AdamW to converge to neither the minimizer nor a KKT point satisfying (\ref{kkt-condtion})\footnote{This does not conflict with \citep{xie-2024-adamw-icml} because \cite{xie-2024-adamw-icml} only considered deterministic AdamW.}. In practical implementations, excessively large values of $\lambda$ are typically avoided, as they may drive the parameters toward zero and away from the minimum solution.

\begin{figure}
\hspace*{2.2cm}\begin{tabular}{@{\extracolsep{1em}}c@{\extracolsep{1em}}c}
   \includegraphics[width=0.35\linewidth,keepaspectratio]{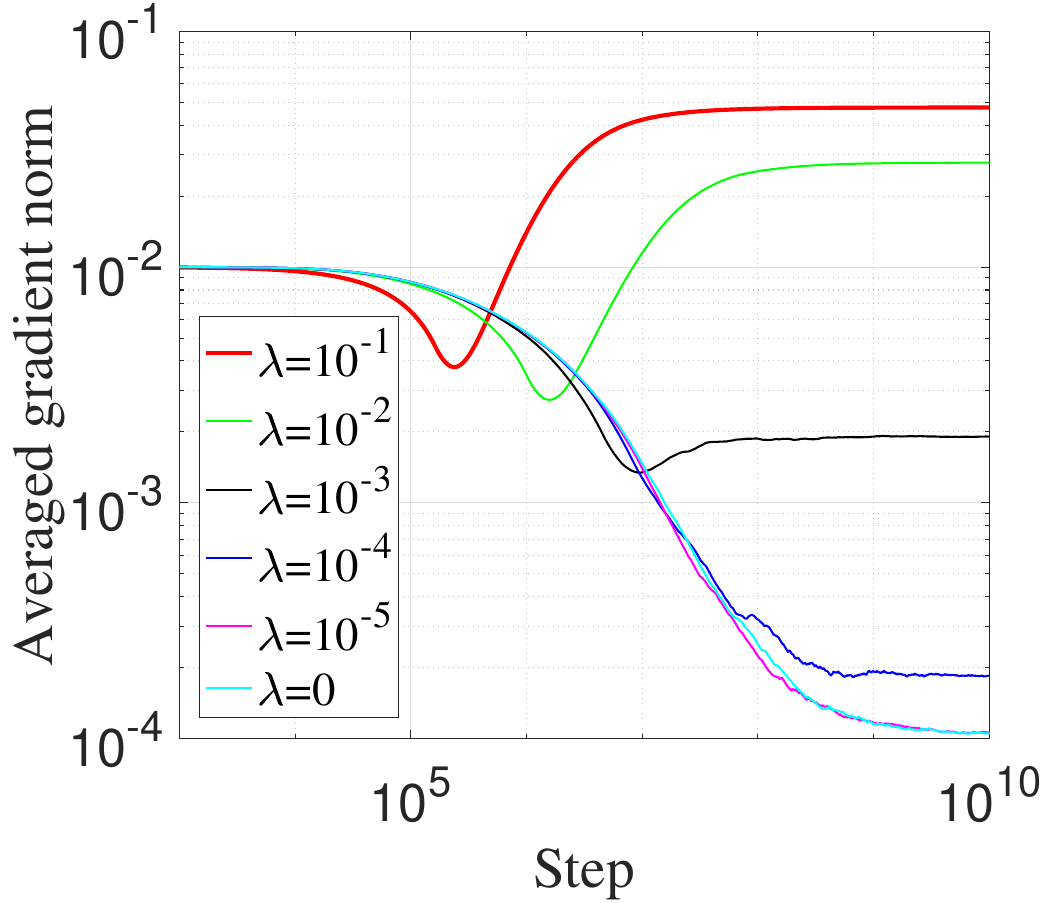}
   &\includegraphics[width=0.34\linewidth,keepaspectratio]{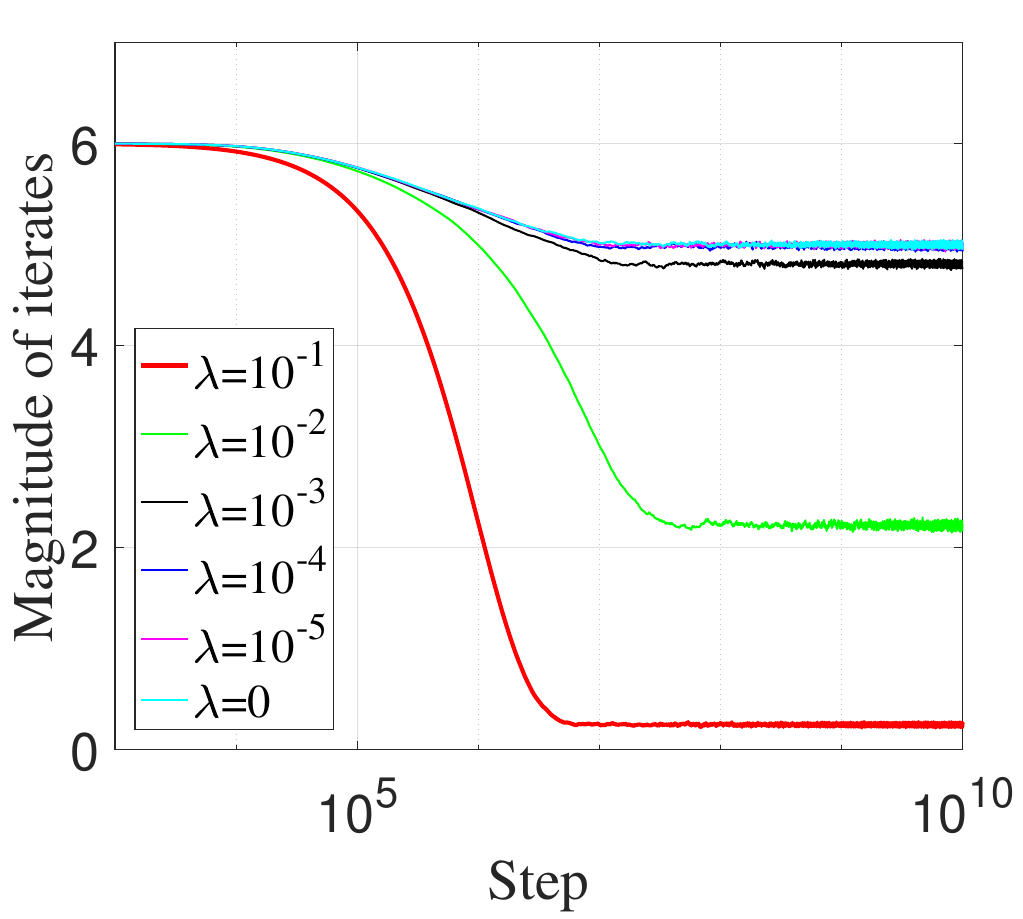}
   \end{tabular}
   \caption{Illustrations of $\frac{1}{k}\sum_{t=1}^k|\nabla f(x^t)|$ (left) and $x^k$ (right) over steps on the toy example.}
   \label{figure5}
\end{figure}

We also initialize $\x^1$ such that $\|\x^1\|_{\infty}\leq \frac{5}{8}\sqrt{\frac{K(f(\x^1)-f^*)}{dL}}$ in our theorem. Although $\sqrt{\frac{K}{d}}$ is typically smaller than 1 in large language model training, it remains not too small. In practical configurations, we often initialize the network weights very small. On the other hand, although we always initialize the scale parameter in BatchNorm/LayerNorm to 1, we do not use weight decay for the scale parameter in practice.

\subsection{Small $\varepsilon$ Setting}

In practice, $\varepsilon$ is typically set to a very small value, for example, approximately $10^{-16}$ in PyTorch implementation\footnote{In the official PyTorch implementation, $\varepsilon$ appeared in a different place in $\frac{\eta}{\sqrt{\v^k}+\varepsilon}$ and $\varepsilon=10^{-8}$, while we use $\frac{\eta}{\sqrt{\v^k+\varepsilon}}$.}, to prevent division by zero while maintaining the adaptive properties of AdamW and Adam. Larger $\varepsilon$ values would make AdamW and Adam behave similarly to SGD, losing its adaptive learning rate adjustment. In Theorem \ref{theorem1}, we set $\varepsilon=\frac{\hat\sigma^2}{d}=\max\left\{\frac{\sigma^2}{d},\frac{L(f(\x^1)-f^*)}{dK\gamma^2}\right\}$, which remains small due to extremely large $d$ and modest $\sigma^2$. We have empirically shown in Figure \ref{figure4} that $\frac{\sigma^2}{d}\approx 10^{-6}$, $10^{-9}$, and $10^{-10}$ in our experiments of ResNet-50 on CIFAR-100 and ImageNet and GPT-2 on OpenWebText, respectively. Intuitively, $\varepsilon$ should be smaller than the square of stochastic gradient at each coordinate, otherwise, $\varepsilon$ would dominate the magnitude of $\v_i^k$ in $\frac{1}{\sqrt{\v_i^k+\varepsilon}}$. Our setting of $\varepsilon$, the coordinate-wise average of gradient noise variance, approximately resides at this critical threshold. Notably, our convergence rates for both AdamW and Adam do not depend on $\varepsilon$ explicitly. In comparison, existing convergence rates for AdamW and Adam in the literature either explicitly depend on $\varepsilon$ or exhibit a higher dependence on the dimension $d$.

\subsection{Impractical Settings of $\eta$, $\theta$, and $\beta$}\label{sec2-5}  

In Theorem \ref{theorem1}, we set the learning rate $\eta$ very small and the parameters $\theta$ and $\beta$ nearly equal to 1 to satisfy the proof requirements. This differs from standard implementations where $(\theta,\beta)=(0.9,0.999)$ is typically used. Although investigating AdamW/Adam's property under realistic configurations represents an important research direction, as it could yield valuable insights for deep learning hyperparameter tuning, the practical configurations may not guarantee convergence. For instance, prior work \citep{luo-2022-nips,Wang-2024-kdd} demonstrates through constructed examples that Adam with common hyperparameters ($\theta=0.9$, $\beta=\{0.999,0.997,0.995,0.993\}$, and $\eta_k=\frac{0.1}{\sqrt{k}}$) fails to converge to stationary points (see \citep[Figuire 2]{Wang-2024-kdd}). On the other hand, empirical evidence from recent studies \citep[Figure 11]{Cutkosky-2024} \citep[Figure 7]{kaiyue-compare-2025} demonstrates that during the training of language models with practical parameter configurations, the gradient norm exhibits little decrease over the training horizon, although the objective function decreases sufficiently.

\subsection{No Conflict with \citep{xie-2024-adamw-icml}}

For sufficiently large weight-decay parameter $\lambda$ where no critical points exist within the constrained domain of problem (\ref{problem}), the KKT conditions (\ref{kkt-condtion}) serve as a natural convergence metric. As $\lambda$ diminishes, problem (\ref{problem}) asymptotically approaches an unconstrained optimization problem, and AdamW reduces to Adam in the limit. There exists a critical threshold beyond which $\|\nabla f(\x)\|_1$ also becomes a viable metric for convergence. Consequently, our results do not conflict with \citep{xie-2024-adamw-icml}.

\section{Extension to NAdamW Incorporating Double-Momentum}

In this section, we extend our theory to NAdamW \citep{Dozat-nadam-iclr-2016}. NAdamW is originally inspired by Nesterov's accelerated gradient method, and after a series of equivalent transformations and approximations, it eventually results in the formulation presented in Algorithm \ref{nadamw}. In contrast to AdamW, NAdamW incorporates a double-momentum mechanism, a strategy that is also employed in LION \citep{lion-2023-nips} and Muon \citep{muon2024}. When $\tau=1$, NAdamW reduces to AdamW.

\begin{theorem}\label{theorem2}
Suppose that Assumptions 1-3 hold. Define $\hat\sigma^2=\max\left\{\sigma^2,\frac{L(f(\x^1)-f^*)}{K\gamma^2}\right\}$ with any constant $\gamma\in(0,1]$. Let $1-\theta=\sqrt{\frac{L(f(\x^1)-f^*)}{K\hat\sigma^2}}$, $\theta\leq\tau\leq 1$, $\theta\leq \beta\leq\sqrt{\theta}$, $\eta=\sqrt{\frac{f(\x^1)-f^*}{4KdL}}$, $\varepsilon=\frac{\hat\sigma^2}{d}$, $\lambda\leq\frac{\sqrt{d}}{10K^{3/4}}\sqrt[4]{\frac{L^3}{\hat\sigma^2(f(\x^1)-f^*)}}$, and $\|\x^1\|_{\infty}\leq \frac{5}{8}\sqrt{\frac{K(f(\x^1)-f^*)}{dL}}$. Then for NAdamW, we have $\|\x^k\|_{\infty}<\frac{1}{\lambda}$ for all $ k=1,2,\cdots,K$ and
\begin{eqnarray}
\begin{aligned}\notag
\frac{1}{K}\sum_{k=1}^K\E\left[\|\nabla f(\x^k)\|_1\right]\leq \frac{10\sqrt{d}}{K^{1/4}}\sqrt[4]{\hat\sigma^2L(f(\x^1)-f^*)} + 66\sqrt{\frac{dL(f(\x^1)-f^*)}{K}}.
\end{aligned}
\end{eqnarray}
In particular, when $\sigma^2\leq\frac{L(f(\x^1)-f^*)}{K\gamma^2}$, we have $1-\theta=\gamma$, $\theta\leq\tau\leq 1$, $\theta\leq \beta\leq\sqrt{\theta}$, $\eta=\sqrt{\frac{f(\x^1)-f^*}{4KdL}}$, $\varepsilon=\frac{L(f(\x^1)-f^*)}{dK\gamma^2}$, $\lambda\leq\frac{1}{10}\sqrt{\frac{dL\gamma}{K(f(\x^1)-f^*)}}$, $\|\x^1\|_{\infty}\leq \frac{5}{8}\sqrt{\frac{K(f(\x^1)-f^*)}{dL}}$, $\|\x^k\|_{\infty}<\frac{1}{\lambda}$, and accordingly
\begin{eqnarray}
\begin{aligned}\notag
\frac{1}{K}\sum_{k=1}^K\E\left[\|\nabla f(\x^k)\|_1\right]\leq 76\sqrt{\frac{dL(f(\x^1)-f^*)}{K\gamma}}.
\end{aligned}
\end{eqnarray}
\end{theorem}

In the following corollary, we also establish the same convergence rate for NAdam, that is, $\lambda=0$ in Algorithm \ref{nadamw}.

\begin{corollary}\label{corollary2}
With the same assumptions and parameter settings of $1-\theta$, $\tau$, $\eta$, and $\varepsilon$ as in Theorem \ref{theorem2}, but only requiring $0\leq\beta\leq 1$ rather than both $\theta\leq \beta\leq\sqrt{\theta}$ and $\|\x^1\|_{\infty}\leq\frac{5}{8}\sqrt{\frac{K(f(\x^1)-f^*)}{dL}}$, we have the same convergence rate for NAdam as established in Theorem \ref{theorem2}. 
\end{corollary}

\section{Literature Comparisons}\label{sec4}

In this section, we compare our theoretical results with representative ones in the literature. A substantial amount of literature exists regarding the convergence analysis of adaptive gradient algorithms, such as \citep{Bottou-2020-jmlr,Kavis-2022-iclr,Faw-2022-colt,WeiChen-2023-colt,Attia-2023-icml} for AdaGrad-norm, \citep{WeiChen-2023-colt,jiang-adagrad-2024,zhangtong-adagrad-2024,Liu-2023-icml,hong-2024-adagrad} for AdaGrad, \citep{Zou-2019-cvpr,bottou-2022-tmlr,luo-2020-iclr,lihuan-rmsprop-2024,xie-rmsprop-iclr-2024} for RMSProp, \citep{wang-adamSGD-2024} for Adam-norm, \citep{Reddi-2018-iclr,Zou-2019-cvpr,bottou-2022-tmlr,tianbaoyang-2021,Chen-2022-jmlr,luo-2022-nips,Wang-2023-nips,Wang-2024-kdd,JunhongLin-2023,hong-2024-adam,Qi-Zhang-2024-rmsporp,haochuanli-2023,Cutkosky-2024-nips,RuinanJin-2025-icml} for Adam, and \citep{Zaheer-2018-nips,chen-2019-iclr,luo-2019-iclr,You-2019-iclr,Zhuang-2020-nips,Chen-2021-ijcai,Savarese-2021-cvpr,Crawshaw-2022,xie-2022,ADOPT-2024-nips} for other variants. We primarily compare with the literature on AdamW and Adam. For Adam, we restrict our comparison to studies with the state-of-the-art convergence rates that do not require the bounded gradient assumption.

\subsection{AdamW: Comparison with \citep{zhou-pami-2024}}

To the best of our knowledge based on a comprehensive literature review, \citep{zhou-pami-2024} appears to be the only existing paper addressing AdamW's convergence and convergence rate. We compare with \citep{zhou-pami-2024} in the following aspects. First, the assumptions in \citep{zhou-pami-2024} are stronger than ours. Denoting $f(\x)=\E_{\zeta\sim \mathcal{P}}[f(\x;\zeta)]$, they assumed $\|\nabla f(\y;\zeta)-\nabla f(\x;\zeta)\|\leq L\|\y-\x\|$ (under which the lower bound is $\bO(\frac{1}{\epsilon^3})$, rather than $\bO(\frac{1}{\epsilon^4})$ \citep{Arjevani-2023-mp}) and $\|\g^k\|_{\infty}\leq c_{\infty}$, while we only assume $\|\nabla f(\y)-\nabla f(\x)\|\leq L\|\y-\x\|$ without the bounded gradient assumption. Second, they set the weight-decay parameter $\lambda_k=\lambda(1-\frac{\beta c_{\infty}^2}{\varepsilon})^k$, which decreases exponentially, making AdamW reduce to standard Adam in the limit. Third, they establish the complexity of $\bO(\max\{\frac{c_{\infty}^{2.5}L\sigma^2(f(\x^1)-f^*)}{\varepsilon^{1.25}\epsilon^4},\frac{c_{\infty}^2\sigma^4}{\varepsilon\epsilon^4}\})$ to achieve $\frac{1}{K}\sum_{k=1}^K\E[\|\nabla F_k(\x^k)\|^2]\leq\epsilon^2$, where $F_k$ is a dynamic $\ell_2$-regularized objective. Their complexity depends on $\varepsilon$ explicitly, which is usually small in practice, for example, $\varepsilon\approx 10^{-16}$ in PyTorch implementation. By contrast, our convergence rate does not depend on $\varepsilon$ explicitly.

\subsection{Adam: Comparison with \citep{lihuan-rmsprop-2024}}\label{sec4-1}

\cite{lihuan-rmsprop-2024} studied RMSProp and its momentum extension, where RMSProp is a special case of Adam by letting $\theta=0$ and $\lambda=0$ in Algorithm \ref{adamw}. The convergence analysis of Adam presents substantially greater challenges than that of RMSProp and we cannot extend the proofs in \citep{lihuan-rmsprop-2024} to Adam. Alternatively, this paper uses a different proof framework to  establish for Adam the same convergence rate achieved by \citep{lihuan-rmsprop-2024}. Moreover, \citet{lihuan-rmsprop-2024} employed a slightly stronger coordinate-wise bounded noise variance condition $\E_k\left[|\g_i^k-\nabla_i f(\x^k)|^2|\F_{k-1}\right]\leq\sigma_i^2$ while we relax it to the general bounded noise variance. As a trade-off, one limitation of our proof is that it relies on a larger value of parameter $\varepsilon$, although $\varepsilon=\frac{\hat\sigma^2}{d}$ is very small in practice. Specifically, under the parameter settings of $\beta=1-\frac{1}{K}$, $\v_i^0=\lambda\max\{\sigma_i^2,\frac{1}{dK}\}$, and $\lambda\geq \frac{\sigma^2}{KL(f(\x^1)-f^*)}$ in \citep{lihuan-rmsprop-2024}, we have $\frac{1}{e^2}\leq\beta^t\leq 1$ for any $t\leq K$ and
\begin{eqnarray}
\begin{aligned}\notag
\v_i^k=\beta^k\v_i^0+(1-\beta)\sum_{t=1}^k\beta^{k-t}|\g_i^t|^2\approx \frac{\sigma_i^2}{K}\frac{\sigma^2}{L(f(\x^1)-f^*)} + \frac{1}{K}\sum_{t=1}^k|\g_i^t|^2,
\end{aligned}
\end{eqnarray}
where $\beta^k\v_i^0$ plays the role of $\varepsilon$ in Algorithm \ref{adamw}, which is of the order $\frac{\sigma_i^2}{K}$, or approximately $\frac{\sigma^2}{dK}$. By contrast, in this paper, we have
\begin{eqnarray}
\begin{aligned}\notag
\v_i^k+\varepsilon=\varepsilon+(1-\beta)\sum_{t=1}^k\beta^{k-t}|\g_i^t|^2=\frac{\hat\sigma^2}{d}+(1-\beta)\sum_{t=1}^k\beta^{k-t}|\g_i^t|^2\approx \sigma_i^2+(1-\beta)\sum_{t=1}^k\beta^{k-t}|\g_i^t|^2.
\end{aligned}
\end{eqnarray}
When $\nabla f(\x^t)\approx 0$ such that $|\g_i^t|\approx \sigma_i$, we have $(1-\beta)\sum_{t=1}^k\beta^{k-t}|\g_i^t|^2\approx \sigma_i^2$. Thus, $\varepsilon$ accounts for nearly half of $(\v_i^k+\varepsilon)$'s size, while in \citep{lihuan-rmsprop-2024}, $\beta^k\v_i^0$ only accounts for close to $\frac{1}{k}$ of $\v_i^k$'s total size. Other representative studies \citep{JunhongLin-2023,hong-2024-adam,Wang-2023-nips} have derived convergence guarantees for Adam built upon weak $\varepsilon$-dependent analysis. However, these results all yield slower convergence rates than ours with a higher dependence on the dimension $d$.

\subsection{Adam: Comparison with \citep{haochuanli-2023}}

\cite{haochuanli-2023} studied Adam under assumption $\|\g^k-\nabla f(\x^k)\|\leq \sigma$ with probability 1 and proved $\frac{1}{K}\sum_{k=1}^K\|\nabla f(\x^k)\|_2^2\leq\epsilon^2$ with high probability within $\bO(\frac{G^{2.5} \sigma^2L(f(\x^1)-f^*)}{\widetilde\varepsilon^{2.5}\epsilon^4})$ iterations. That is, $\frac{1}{K}\sum_{k=1}^K\|\nabla f(\x^k)\|_2\leq (\frac{G}{\widetilde\varepsilon})^{5/8}\frac{1}{K^{1/4}}\sqrt[4]{\sigma^2L(f(\x^1)-f^*)}$, where $G\geq \max\{\widetilde\varepsilon,\sigma,\sqrt{L(f(\x^1)-f^*)}\}$ and $\widetilde\varepsilon$ appeared in a different place in $\frac{\m^k}{\sqrt{\v^k}+\widetilde\varepsilon}$ (hence we may consider $\widetilde\varepsilon$ to be equal to $\sqrt{\varepsilon}$). When $\|\nabla f(\x)\|_1=\varTheta(\sqrt{d})\|\nabla f(\x)\|_2$, as empirically observed in real-world deep learning training, our convergence rate is $(\frac{G}{\widetilde\varepsilon})^{5/8}$ times faster than \citep{haochuanli-2023}. In PyTorch implementation, the default value of $\widetilde\varepsilon$ is typically set to $10^{-8}$. To eliminate the dependence on $\varepsilon$, \cite{haochuanli-2023} require $\widetilde\varepsilon^2(\approx \varepsilon)=G^2\geq\max\{\sigma^2,L(f(\x^1)-f^*)\}\geq \sigma^2$, while we only need $\varepsilon=\frac{\sigma^2}{d}$, which is $d$ times smaller. 

\section{Proof of Theorem \ref{theorem2}}\label{sec:proof}
In this section, we prove Theorem \ref{theorem2} with the supporting lemmas in Section \ref{sec-supporting-lemmas}. Theorem \ref{theorem1} is a special case of Theorem \ref{theorem2} by letting $\tau=1$ and we omit the details.
\begin{proof}
As the gradient is $L$-Lipschitz, we have
\begin{eqnarray}
\begin{aligned}\label{equ1}
&f(\x^{k+1})-f(\x^k)\\
\leq& \<\nabla f(\x^k),\x^{k+1}-\x^k\>+\frac{L}{2}\|\x^{k+1}-\x^k\|^2\\
=& -\eta\sum_{i=1}^d\<\nabla_i f(\x^k),\frac{\widetilde\m_i^k+\lambda\x_i^k\sqrt{\v_i^k+\varepsilon}}{\sqrt{\v_i^k+\varepsilon}}\>+\frac{L\eta^2}{2}\sum_{i=1}^d\frac{\left|\widetilde\m_i^k+\lambda\x_i^k\sqrt{\v_i^k+\varepsilon}\right|^2}{\v_i^k+\varepsilon} \\
=& -\frac{\eta}{2}\sum_{i=1}^d\frac{\left|\nabla_i f(\x^k)\right|^2}{\sqrt{\v_i^k+\varepsilon}}-\frac{\eta}{2}\sum_{i=1}^d\frac{\left|\widetilde\m_i^k+\lambda\x_i^k\sqrt{\v_i^k+\varepsilon}\right|^2}{\sqrt{\v_i^k+\varepsilon}}\\
&+\frac{\eta}{2}\sum_{i=1}^d\frac{\left|\nabla_i f(\x^k)-\widetilde\m_i^k-\lambda\x_i^k\sqrt{\v_i^k+\varepsilon}\right|^2}{\sqrt{\v_i^k+\varepsilon}}+\frac{L\eta^2}{2}\sum_{i=1}^d\frac{\left|\widetilde\m_i^k+\lambda\x_i^k\sqrt{\v_i^k+\varepsilon}\right|^2}{\v_i^k+\varepsilon}\\
\leq& -\frac{\eta}{2}\sum_{i=1}^d\frac{\left|\nabla_i f(\x^k)\right|^2}{\sqrt{\v_i^k+\varepsilon}}-\frac{\eta}{2}\sum_{i=1}^d\frac{\left|\widetilde\m_i^k+\lambda\x_i^k\sqrt{\v_i^k+\varepsilon}\right|^2}{\sqrt{\v_i^k+\varepsilon}}\\
&+\eta\sum_{i=1}^d\frac{\left|\nabla_i f(\x^k)-\widetilde\m_i^k\right|^2+\left|\lambda\x_i^k\sqrt{\v_i^k+\varepsilon}\right|^2}{\sqrt{\v_i^k+\varepsilon}}+\frac{L\eta^2}{2\sqrt{\varepsilon}}\sum_{i=1}^d\frac{\left|\widetilde\m_i^k+\lambda\x_i^k\sqrt{\v_i^k+\varepsilon}\right|^2}{\sqrt{\v_i^k+\varepsilon}}\\
\overset{a}\leq& -\frac{\eta}{2}\sum_{i=1}^d\frac{\left|\nabla_i f(\x^k)\right|^2}{\sqrt{\v_i^k+\varepsilon}}-\frac{\eta}{4}\sum_{i=1}^d\frac{\left|\widetilde\m_i^k+\lambda\x_i^k\sqrt{\v_i^k+\varepsilon}\right|^2}{\sqrt{\v_i^k+\varepsilon}}\\
&+\frac{\eta}{\sqrt{\varepsilon}}\left\|\nabla f(\x^k)-\widetilde\m^k\right\|^2+\eta\sum_{i=1}^d|\lambda\x_i^k|^2\sqrt{\v_i^k+\varepsilon},
\end{aligned}
\end{eqnarray}
where we let $\eta\leq \frac{\sqrt{\varepsilon}}{2L}$ in $\overset{a}\leq$. Taking conditional expectation with respect to $\g^k$ given $\F_{k-1}$, plugging (\ref{l4-equ2}) in lemma \ref{lemma3} into the above inequality, then for any $k\geq 2$, we have
\begin{eqnarray}
\begin{aligned}\label{equ2}
&\E_k\left[f(\x^{k+1})|\F_{k-1}\right]-f(\x^k)\\
\leq& \E_k\left[-\frac{\eta}{2}\sum_{i=1}^d\frac{\left|\nabla_i f(\x^k)\right|^2}{\sqrt{\v_i^k+\varepsilon}}-\frac{\eta}{4}\sum_{i=1}^d\frac{\left|\widetilde\m_i^k+\lambda\x_i^k\sqrt{\v_i^k+\varepsilon}\right|^2}{\sqrt{\v_i^k+\varepsilon}}+\frac{\eta\theta}{\sqrt{\varepsilon}}\left\|\m^{k-1}-\nabla f(\x^{k-1})\right\|^2\right.\\
&\qquad\left.+\frac{L^2\eta^3}{\varepsilon(1-\theta)}\hspace*{-0.05cm}\sum_{i=1}^d\hspace*{-0.05cm}\frac{\left|\widetilde\m_i^{k-1}\hspace*{-0.05cm}+\hspace*{-0.05cm}\lambda\x_i^{k-1}\hspace*{-0.05cm}\sqrt{\v_i^{k-1}\hspace*{-0.05cm}+\hspace*{-0.05cm}\varepsilon}\right|^2}{\sqrt{\v_i^{k-1}+\varepsilon}}\hspace*{-0.05cm}+\hspace*{-0.05cm}\frac{2\eta(1\hspace*{-0.05cm}-\hspace*{-0.05cm}\theta)}{\sqrt{\varepsilon}}\sigma^2\hspace*{-0.05cm}+\hspace*{-0.05cm}\eta\hspace*{-0.05cm}\sum_{i=1}^d\hspace*{-0.05cm}|\lambda\x_i^k|^2\hspace*{-0.05cm}\sqrt{\v_i^k\hspace*{-0.05cm}+\hspace*{-0.05cm}\varepsilon}\Big|\F_{k-1}\hspace*{-0.05cm}\right]\hspace*{-0.1cm}.\hspace*{-0.5cm}
\end{aligned}
\end{eqnarray}
Multiplying both sides of (\ref{l4-equ1}) in Lemma \ref{lemma3} by $\frac{\eta\theta}{\sqrt{\varepsilon}(1-\theta)}$, adding it to (\ref{equ2}), and rearranging the terms, then for any $k\geq2$, we have 
\begin{eqnarray}
\begin{aligned}\label{equ3}
&\E_k\left[f(\x^{k+1})-f^* +\frac{\eta\theta}{\sqrt{\varepsilon}(1-\theta)}\left\|\nabla f(\x^k)-\m^k\right\|^2+\frac{\eta}{4}\sum_{i=1}^d\frac{\left|\widetilde\m_i^k+\lambda\x_i^k\sqrt{\v_i^k+\varepsilon}\right|^2}{\sqrt{\v_i^k+\varepsilon}}\Big|\F_{k-1}\right]\hspace*{-0.8cm}\\
\leq& f(\x^k)-f^*+\E_k\left[-\frac{\eta}{2}\sum_{i=1}^d\frac{\left|\nabla_i f(\x^k)\right|^2}{\sqrt{\v_i^k+\varepsilon}}+\eta\sum_{i=1}^d|\lambda\x_i^k|^2\sqrt{\v_i^k+\varepsilon}\Big|\F_{k-1}\right]\\
&+\left(\frac{\eta\theta}{\sqrt{\varepsilon}}+\frac{\eta\theta^2}{\sqrt{\varepsilon}(1-\theta)}\right)\left\|\nabla f(\x^{k-1})-\m^{k-1}\right\|^2+\frac{2\eta(1-\theta)\sigma^2}{\sqrt{\varepsilon}}\\
&+\left(\frac{L^2\eta^3}{\varepsilon(1-\theta)}+\frac{L^2\eta^3\theta}{\varepsilon(1-\theta)^2}\right)\sum_{i=1}^d\frac{\left|\widetilde\m_i^{k-1}+\lambda\x_i^{k-1}\sqrt{\v_i^{k-1}+\varepsilon}\right|^2}{\sqrt{\v_i^{k-1}+\varepsilon}} +\frac{\eta\theta(1-\theta)\sigma^2}{\sqrt{\varepsilon}}\hspace*{-0.7cm}\\
\overset{b}\leq& f(\x^k)-f^*+\E_k\left[\underbrace{-\frac{\eta}{2}\sum_{i=1}^d\frac{\left|\nabla_i f(\x^k)\right|^2}{\sqrt{\v_i^k+\varepsilon}}}_{\text{\rm term (a)}}+\underbrace{\eta\sum_{i=1}^d|\lambda\x_i^k|^2\sqrt{\v_i^k+\varepsilon}}_{\text{\rm term (b)}}\Big|\F_{k-1}\right]\\
&+\frac{\eta\theta}{\sqrt{\varepsilon}(1-\theta)}\left\|\nabla f(\x^{k-1})-\m^{k-1}\right\|^2+\frac{\eta}{4}\sum_{i=1}^d\frac{\left|\widetilde\m_i^{k-1}+\lambda\x_i^{k-1}\sqrt{\v_i^{k-1}+\varepsilon}\right|^2}{\sqrt{\v_i^{k-1}+\varepsilon}}+\frac{3\eta(1-\theta)\sigma^2}{\sqrt{\varepsilon}},\hspace*{-0.7cm}
\end{aligned}
\end{eqnarray}
where we let $\eta^2\leq\frac{\varepsilon(1-\theta)^2}{4L^2}$ such that $\frac{L^2\eta^3}{\varepsilon(1-\theta)}+\frac{L^2\eta^3\theta}{\varepsilon(1-\theta)^2}=\frac{L^2\eta^3}{\varepsilon(1-\theta)^2}\leq\frac{\eta}{4}$ in $\overset{b}\leq$. 

We can recursively eliminate certain terms in the above inequality after telescoping, except for the problematic term (b). The following outlines the key technical components of our proof to address term (b) and achieve tight convergence rate. 

\noindent\textbf{(1). Bounding $|\lambda\x_i^k|$ in Term (b) by $\frac{\sqrt{\nu}}{K^{1/4}}$ for Some Constant $\nu$}.

\noindent From Lemma \ref{lemma2}, we have
\begin{eqnarray}
\begin{aligned}\notag
\|\x^{k+1}\|_{\infty}-\frac{3}{\lambda}\leq (1-\eta\lambda)^k\left(\|\x^1\|_{\infty}-\frac{3}{\lambda}\right).
\end{aligned}
\end{eqnarray}
Since $\ln x\leq x-1$ and $e^x\geq x+1$ for any $x>0$, letting $\eta\lambda\leq\frac{\sqrt{\nu}}{5K^{5/4}}$ and $\frac{\sqrt{\nu}}{K^{1/4}}<1$, we have for any $k\leq K$ that
\begin{eqnarray}
\begin{aligned}\notag
k\ln(1-\eta\lambda)=-k\ln\frac{1}{1-\eta\lambda}\geq -K\left(\frac{1}{1-\eta\lambda}-1\right)=-\frac{K\eta\lambda}{1-\eta\lambda}\geq -\frac{\sqrt{\nu}}{4K^{1/4}}
\end{aligned}
\end{eqnarray}
and
\begin{eqnarray}
\begin{aligned}\notag
(1-\eta\lambda)^k\geq e^{-\frac{\sqrt{\nu}}{4K^{1/4}}}\geq 1-\frac{\sqrt{\nu}}{4K^{1/4}}.
\end{aligned}
\end{eqnarray}
Initializing $\|\x^1\|_{\infty}\leq\frac{\sqrt{\nu}}{4K^{1/4}\lambda}$, we have
\begin{eqnarray}
\begin{aligned}\notag
\|\x^{k+1}\|_{\infty}-\frac{3}{\lambda}\leq& -\frac{1}{\lambda}(1-\eta\lambda)^k\left(3-\frac{\sqrt{\nu}}{4K^{1/4}}\right)\\
\leq& -\frac{1}{\lambda}\left(1-\frac{\sqrt{\nu}}{4K^{1/4}}\right)\left(3-\frac{\sqrt{\nu}}{4K^{1/4}}\right)\leq-\frac{3}{\lambda}+\frac{1}{\lambda}\frac{\sqrt{\nu}}{K^{1/4}},
\end{aligned}
\end{eqnarray}
leading to $\lambda\|\x^k\|_{\infty}\leq \frac{\sqrt{\nu}}{K^{1/4}}$ for any $k=1,2,\cdots,K$.

\noindent\textbf{(2). Relaxing Term (b) and Absorbing it within Term (a)}.

\noindent Define $\widetilde\v_i^k$ recursively as follows
\begin{eqnarray}
\begin{aligned}\label{wide-v-def}
\widetilde\v_i^k=\beta\widetilde\v_i^{k-1}+2(1-\beta)\left(|\nabla_i f(\x^k)|^2+|\g_i^k-\nabla_i f(\x^k)|^2\right)\quad \mbox{with}\quad \widetilde\v_i^0=0.
\end{aligned}
\end{eqnarray}
From the recursion of $\v_i^k$ and $\v_i^0=0$, by induction, we have
\begin{eqnarray}
\begin{aligned}\label{equ8}
\v_i^k\leq\widetilde\v_i^k,\quad \sqrt{\v_i^k+\varepsilon}\leq \sqrt{\widetilde\v_i^k+\varepsilon},\quad-\frac{\left|\nabla_i f(\x^k)\right|^2}{\sqrt{\v_i^k+\varepsilon}}\leq -\frac{\left|\nabla_i f(\x^k)\right|^2}{\sqrt{\widetilde\v_i^k+\varepsilon}}.
\end{aligned}
\end{eqnarray}
First using $\lambda\|\x^k\|_{\infty}\leq \frac{\sqrt{\nu}}{K^{1/4}}$ and then plugging (\ref{equ8}) into (\ref{equ1}) and (\ref{equ3}), we have
\begin{eqnarray}
\begin{aligned}\label{equ4}
&\E_k\left[f(\x^{k+1})-f^* + \frac{\eta}{4}\sum_{i=1}^d\frac{\left|\widetilde\m_i^k+\lambda\x_i^k\sqrt{\v_i^k+\varepsilon}\right|^2}{\sqrt{\v_i^k+\varepsilon}}  \Big|\F_{k-1}\right]\\
\leq& f(\x^k)\hspace*{-0.03cm}-\hspace*{-0.03cm}f^*\hspace*{-0.03cm}+\hspace*{-0.03cm}\E_k\hspace*{-0.12cm}\left[\underbrace{-\frac{\eta}{2}\hspace*{-0.03cm}\sum_{i=1}^d\hspace*{-0.03cm}\frac{\left|\nabla_i f(\x^k)\right|^2}{\sqrt{\widetilde\v_i^k\hspace*{-0.03cm}+\hspace*{-0.03cm}\varepsilon}}}_{\text{\rm term (c)}}\hspace*{-0.03cm}+\hspace*{-0.03cm}\underbrace{\frac{\eta\nu}{K^{1/2}}\hspace*{-0.03cm}\sum_{i=1}^d\hspace*{-0.03cm}\sqrt{\widetilde\v_i^k\hspace*{-0.03cm}+\hspace*{-0.03cm}\varepsilon}}_{\text{\rm term (d)}} + \frac{\eta}{\sqrt{\varepsilon}}\left\|\nabla f(\x^k)\hspace*{-0.03cm}-\hspace*{-0.03cm}\widetilde\m^k\right\|^2\hspace*{-0.03cm}\Big|\F_{k-1}\right]
\end{aligned}
\end{eqnarray}
and
\begin{eqnarray}
\begin{aligned}\label{equ5}
&\E_k\left[f(\x^{k+1})-f^* +\frac{\eta\theta}{\sqrt{\varepsilon}(1-\theta)}\left\|\nabla f(\x^k)-\m^k\right\|^2+\frac{\eta}{4}\sum_{i=1}^d\frac{\left|\widetilde\m_i^k+\lambda\x_i^k\sqrt{\v_i^k+\varepsilon}\right|^2}{\sqrt{\v_i^k+\varepsilon}}\Big|\F_{k-1}\right]\hspace*{-0.5cm}\\
\leq& f(\x^k)-f^*\underbrace{-\frac{\eta}{2}\sum_{i=1}^d\E_k\left[\frac{\left|\nabla_i f(\x^k)\right|^2}{\sqrt{\widetilde\v_i^k+\varepsilon}}\Big|\F_{k-1}\right]}_{\text{\rm term (c)}}+\underbrace{\frac{\eta\nu}{K^{1/2}}\sum_{i=1}^d\E_k\left[\sqrt{\widetilde\v_i^k+\varepsilon}\Big|\F_{k-1}\right]}_{\text{\rm term (d)}}\\
&+\frac{\eta\theta}{\sqrt{\varepsilon}(1-\theta)}\left\|\nabla f(\x^{k-1})-\m^{k-1}\right\|^2+\frac{\eta}{4}\sum_{i=1}^d\frac{\left|\widetilde\m_i^{k-1}+\lambda\x_i^{k-1}\sqrt{\v_i^{k-1}+\varepsilon}\right|^2}{\sqrt{\v_i^{k-1}+\varepsilon}}+\frac{3\eta(1-\theta)\sigma^2}{\sqrt{\varepsilon}}.\hspace*{-0.7cm}
\end{aligned}
\end{eqnarray} 
For both (\ref{equ4}) and (\ref{equ5}), taking expectation with respect to $\F_{k-1}$ and summing (\ref{equ4}) with $k=1$ and (\ref{equ5}) over $k=2,3,\cdots,K$, we have
\begin{eqnarray}
\begin{aligned}\label{equ6}
&\E_{\F_K}\left[f(\x^{K+1})-f^* + \frac{\eta\theta}{\sqrt{\varepsilon}(1-\theta)}\left\|\nabla f(\x^K)-\m^K\right\|^2 + \frac{\eta}{4}\sum_{i=1}^d\frac{\left|\widetilde\m_i^K+\lambda\x_i^K\sqrt{\v_i^K+\varepsilon}\right|^2}{\sqrt{\v_i^K+\varepsilon}} \right]\hspace*{-2cm}\\
\leq& f(\x^1)-f^*-\frac{\eta}{2}\sum_{k=1}^K\sum_{i=1}^d\E_{\F_k}\left[\frac{\left|\nabla_i f(\x^k)\right|^2}{\sqrt{\widetilde\v_i^k+\varepsilon}}\right]+\frac{\eta\nu}{K^{1/2}}\sum_{k=1}^K\sum_{i=1}^d\E_{\F_k}\left[\sqrt{\widetilde\v_i^k+\varepsilon}\right]\\
&+\frac{\eta\theta}{\sqrt{\varepsilon}(1-\theta)}\E_{\F_1}\left[\|\nabla f(\x^1)-\m^1\|^2\right]+ \frac{\eta}{\sqrt{\varepsilon}}\E_{\F_1}\left[\|\nabla f(\x^1)-\widetilde\m^1\|^2\right]+\frac{3(K-1)\eta(1-\theta)\sigma^2}{\sqrt{\varepsilon}}.\hspace*{-0.6cm}
\end{aligned}
\end{eqnarray} 
As the gradient is $L$-Lipschitz, we have
\begin{eqnarray}
\begin{aligned}\notag
f^*\hspace*{-0.02cm}\leq\hspace*{-0.02cm} f\left(\x\hspace*{-0.02cm}-\hspace*{-0.02cm}\frac{1}{L}\nabla f(\x)\right)\hspace*{-0.02cm}\leq\hspace*{-0.02cm} f(\x)\hspace*{-0.02cm}-\hspace*{-0.02cm}\frac{1}{L}\<\nabla f(\x),\nabla f(\x)\>\hspace*{-0.02cm}+\hspace*{-0.02cm}\frac{L}{2}\left\|\frac{1}{L}\nabla f(\x)\right\|^2\hspace*{-0.02cm}=\hspace*{-0.02cm}f(\x)\hspace*{-0.02cm}-\hspace*{-0.02cm}\frac{1}{2L}\|\nabla f(\x)\|^2.
\end{aligned}
\end{eqnarray}
Using the recursions of $\m^1$ and $\widetilde\m^1$, $\m^0=0$, and $0\leq\theta\leq\tau\leq 1$, we have
\begin{eqnarray}
\begin{aligned}\notag
\E_{\F_1}\left[\|\nabla f(\x^1)-\m^1\|^2\right]=&\E_{\F_1}\left[\|\theta\nabla f(\x^1)+(1-\theta)(\nabla f(\x^1)-\g^1)\|^2\right]\\
=& \theta^2\|\nabla f(\x^1)\|^2 + (1-\theta)^2\E_{\F_1}\left[\|\nabla f(\x^1)-\g^1\|^2\right]\\
\leq& 2L(f(\x^1)-f^*) + (1-\theta)^2\sigma^2
\end{aligned}
\end{eqnarray}
and
\begin{eqnarray}
\begin{aligned}\notag
\E_{\F_1}\left[\|\nabla f(\x^1)-\widetilde\m^1\|^2\right]=&\E_{\F_1}\left[\|\tau\theta\nabla f(\x^1)+(1-\tau\theta)(\nabla f(\x^1)-\g^1)\|^2\right]\\
=& \tau^2\theta^2\|\nabla f(\x^1)\|^2 + (1-\tau\theta)^2\E_{\F_1}\left[\|\nabla f(\x^1)-\g^1\|^2\right]\\
\leq& 2L(f(\x^1)-f^*) + (1-\theta^2)^2\sigma^2\\
\leq& 2L(f(\x^1)-f^*) + 4(1-\theta)^2\sigma^2.
\end{aligned}
\end{eqnarray}
Plugging into (\ref{equ6}), we have
\begin{eqnarray}
\begin{aligned}\label{equ7}
&\E_{\F_K}\left[f(\x^{K+1})-f^* + \frac{\eta\theta}{\sqrt{\varepsilon}(1-\theta)}\left\|\nabla f(\x^K)-\m^K\right\|^2 + \frac{\eta}{4}\sum_{i=1}^d\frac{\left|\widetilde\m_i^K+\lambda\x_i^K\sqrt{\v_i^K+\varepsilon}\right|^2}{\sqrt{\v_i^K+\varepsilon}} \right]\hspace*{-2cm}\\
\leq& f(\x^1)-f^*-\frac{\eta}{2}\sum_{k=1}^K\sum_{i=1}^d\E_{\F_k}\left[\frac{\left|\nabla_i f(\x^k)\right|^2}{\sqrt{\widetilde\v_i^k+\varepsilon}}\right]+\frac{\eta\nu}{K^{1/2}}\sum_{k=1}^K\sum_{i=1}^d\E_{\F_k}\left[\sqrt{\widetilde\v_i^k+\varepsilon}\right]\\
&+\frac{\eta}{\sqrt{\varepsilon}(1-\theta)}\left(2L(f(\x^1)-f^*) + 4(1-\theta)^2\sigma^2\right)+\frac{3(K-1)\eta(1-\theta)\sigma^2}{\sqrt{\varepsilon}}\\
\overset{c}\leq& f(\x^1)-f^*\underbrace{-\frac{\eta}{2}\sum_{k=1}^K\sum_{i=1}^d\E_{\F_k}\left[\frac{\left|\nabla_i f(\x^k)\right|^2}{\sqrt{\widetilde\v_i^k+\varepsilon}}\right]}_{\text{\rm term (e)}}+\frac{2\eta L(f(\x^1)-f^*)}{\sqrt{\varepsilon}(1-\theta)} +\frac{4K\eta(1-\theta)\sigma^2}{\sqrt{\varepsilon}}\\
&+\frac{\eta\nu}{K^{1/2}}\left(\underbrace{\sqrt{2d}K\sigma+Kd\sqrt{\varepsilon}}_{\text{\rm term (f)}}+\underbrace{4\sum_{k=1}^K\sum_{i=1}^d\E_{\F_k}\left[\frac{|\nabla_i f(\x^k)|^2}{\sqrt{\widetilde\v_i^k+\varepsilon}}\right]}_{\text{\rm term (g)}}\right)\\
\overset{d}\leq& f(\x^1)\hspace*{-0.06cm}-\hspace*{-0.06cm}f^*\hspace*{-0.06cm}-\hspace*{-0.06cm}\frac{\eta}{4}\hspace*{-0.06cm}\sum_{k=1}^K\hspace*{-0.06cm}\sum_{i=1}^d\hspace*{-0.06cm}\E_{\F_k}\hspace*{-0.06cm}\left[\frac{\left|\nabla_i f(\x^k)\right|^2}{\sqrt{\widetilde\v_i^k+\varepsilon}}\right]\hspace*{-0.06cm}+\hspace*{-0.06cm}\frac{2\eta L(f(\x^1)\hspace*{-0.06cm}-\hspace*{-0.06cm}f^*)}{\sqrt{\varepsilon}(1-\theta)} \hspace*{-0.06cm}+\hspace*{-0.06cm}\frac{4K\eta(1\hspace*{-0.06cm}-\hspace*{-0.06cm}\theta)\sigma^2}{\sqrt{\varepsilon}}\hspace*{-0.06cm}+\hspace*{-0.06cm}2.5\eta\nu d\sqrt{K\varepsilon},\hspace*{-0.7cm}
\end{aligned}
\end{eqnarray}
where we use Lemma \ref{lemma4} in $\overset{c}\leq$, let $\frac{\nu}{K^{1/2}}\leq \frac{1}{16}$ and $\varepsilon\geq\frac{\sigma^2}{d}$ such that $\sqrt{2d}K\sigma\leq \sqrt{2}Kd\sqrt{\varepsilon}$ in $\overset{d}\leq$.
 
So from (\ref{equ7}), we have
\begin{eqnarray}
\begin{aligned}\label{equ9}
&\sum_{k=1}^K\sum_{i=1}^d\E_{\F_k}\left[\frac{\left|\nabla_i f(\x^k)\right|^2}{\sqrt{\widetilde\v_i^k+\varepsilon}}\right]\\
\leq& \frac{4(f(\x^1)-f^*)}{\eta}+\frac{8L( f(\x^1)-f^* )}{\sqrt{\varepsilon}(1-\theta)}+\frac{16K(1-\theta)\sigma^2}{\sqrt{\varepsilon}}+10\nu d\sqrt{K\varepsilon}\\
\leq& \frac{4(f(\x^1)-f^*)}{\eta}+\frac{8L( f(\x^1)-f^* )}{\sqrt{\varepsilon}(1-\theta)}+\frac{16K(1-\theta)\hat\sigma^2}{\sqrt{\varepsilon}}+10\nu d\sqrt{K\varepsilon},
\end{aligned}
\end{eqnarray}
where we denote $\hat\sigma^2=\max\left\{\sigma^2,\frac{L(f(\x^1)-f^*)}{K\gamma^2}\right\}$ with any constant $\gamma\in(0,1]$ to incorporate the scenario when $\sigma^2\leq\frac{L(f(\x^1)-f^*)}{K}$.

Recall that we require the parameters satisfying the following relations in the above proof:
\begin{eqnarray}
\begin{aligned}\notag
\eta\leq \frac{\sqrt{\varepsilon}}{2L},\hspace*{0.15cm}\eta^2\leq\frac{\varepsilon(1-\theta)^2}{4L^2},\hspace*{0.15cm}\eta\lambda\leq\frac{\sqrt{\nu}}{5K^{5/4}},\hspace*{0.15cm} \frac{\sqrt{\nu}}{K^{1/4}}<1,\hspace*{0.15cm}\|\x^1\|_{\infty}\leq\frac{\sqrt{\nu}}{4K^{1/4}\lambda},\hspace*{0.15cm}\frac{\nu}{K^{1/2}}\leq \frac{1}{16},\hspace*{0.15cm}\varepsilon\geq\frac{\sigma^2}{d}.
 \end{aligned}
\end{eqnarray}
Recalling $\hat\sigma^2\geq \frac{L(f(\x^1)-f^*)}{K\gamma^2}$ and letting $\varepsilon=\frac{\hat\sigma^2}{d}$, $1-\theta=\sqrt{\frac{L(f(\x^1)-f^*)}{K\hat\sigma^2}}$, $\eta=\sqrt{\frac{\varepsilon(f(\x^1)-f^*)}{4K\hat\sigma^2 L}}=\sqrt{\frac{f(\x^1)-f^*}{4KdL}}$, $\nu=\frac{1}{16d}\sqrt{\frac{\hat\sigma^2L(f(\x^1)-f^*)}{\varepsilon^2}}=\frac{1}{16}\sqrt{\frac{L(f(\x^1)-f^*)}{\hat\sigma^2}}$, $\lambda\leq\frac{\sqrt{\nu}}{5K^{5/4}\eta}=\frac{\sqrt{ d}}{10K^{3/4}}\sqrt[4]{\frac{L^3}{\hat\sigma^2(f(\x^1)-f^*)}}$, and $\|\x^1\|_{\infty}\leq\frac{5}{8}\sqrt{\frac{K(f(\x^1)-f^*)}{dL}}=\frac{5K\eta}{4}\leq\frac{\sqrt{\nu}}{4K^{1/4}\lambda}$, the above requirements are satisfied. So we have from (\ref{equ9}) that
\begin{eqnarray}
\begin{aligned}\label{equ10}
&\sum_{k=1}^K\sum_{i=1}^d\E_{\F_k}\left[\frac{\left|\nabla_i f(\x^k)\right|^2}{\sqrt{\widetilde\v_i^k+\varepsilon}}\right]\leq33\sqrt{\frac{K\hat\sigma^2L(f(\x^1)-f^*)}{\varepsilon}}.
\end{aligned}
\end{eqnarray}

\noindent\textbf{(3). Eliminating $\varepsilon$ in the Final Convergence Rate}.

\noindent Using H\"older's inequality, Lemma \ref{lemma4}, and (\ref{equ10}), we have
\begin{eqnarray}
\begin{aligned}\notag
&\left(\sum_{k=1}^K\E_{\F_k}\left[\|\nabla f(\x^k)\|_1\right]\right)^2\\
\leq&\left(\sum_{k=1}^K\sum_{i=1}^d\E_{\F_k}\left[\frac{\left|\nabla_i f(\x^k)\right|^2}{\sqrt{\widetilde\v_i^k+\varepsilon}}\right]\right)\left(\sum_{k=1}^K\sum_{i=1}^d\E_{\F_k}\left[\sqrt{\widetilde\v_i^k+\varepsilon}\right]\right)\\
\leq&\left(\sum_{k=1}^K\sum_{i=1}^d\E_{\F_k}\left[\frac{\left|\nabla_i f(\x^k)\right|^2}{\sqrt{\widetilde\v_i^k+\varepsilon}}\right]\right)\left(\sqrt{2d}K\sigma+Kd\sqrt{\varepsilon}+4\sum_{k=1}^K\sum_{i=1}^d\E_{\F_k}\left[\frac{\left|\nabla_i f(\x^k)\right|^2}{\sqrt{\widetilde\v_i^k+\varepsilon}}\right]\right)\\
\leq&\left(33\sqrt{\frac{K\hat\sigma^2L(f(\x^1)-f^*)}{\varepsilon}}\right)\left(132\sqrt{\frac{K\hat\sigma^2L(f(\x^1)-f^*)}{\varepsilon}}+\sqrt{2d}K\sigma+Kd\sqrt{\varepsilon}\right)
\end{aligned}
\end{eqnarray}
and
\begin{eqnarray}
\begin{aligned}\notag
&\frac{1}{K}\sum_{k=1}^K\E_{\F_k}\left[\|\nabla f(\x^k)\|_1\right]\\
\leq&\frac{1}{K}\left(66\sqrt{\frac{K\hat\sigma^2L(f(\x^1)-f^*)}{\varepsilon}}+6\sqrt[4]{\underbrace{\frac{K\hat\sigma^2L(f(\x^1)-f^*)}{\varepsilon}\left(\sqrt{2d}K\sigma+Kd\sqrt{\varepsilon}\right)^2}_{\text{\rm term (h)}}}\right).
\end{aligned}
\end{eqnarray}
The above convergence rate is not optimal due to its explicit dependence on $\varepsilon$, which is absent from the optimal rate (\ref{SGD-rate}) of SGD. By setting $\varepsilon=\frac{\hat\sigma^2}{d}$, we obtain $\left(\sqrt{2d}K\sigma+Kd\sqrt{\varepsilon}\right)^2\leq 7K^2d^2\varepsilon$, which allows us to eliminate $\varepsilon$ in the denominator of term (h). This yields the following final convergence rate
\begin{eqnarray}
\begin{aligned}\notag
\frac{1}{K}\sum_{k=1}^K\E_{\F_k}\left[\|\nabla f(\x^k)\|_1\right]\leq& 66\sqrt{\frac{dL(f(\x^1)-f^*)}{K}}+\frac{10\sqrt{d}}{K^{1/4}}\sqrt[4]{\hat\sigma^2L(f(\x^1)-f^*)}.
\end{aligned}
\end{eqnarray}
Finally, from $\lambda\|\x^k\|_{\infty}\leq \frac{\sqrt{\nu}}{K^{1/4}}$ and the settings of $\nu$ and $\hat\sigma$, we have
\begin{eqnarray}
\begin{aligned}\notag
&\lambda\|\x^k\|_{\infty}\leq \frac{\sqrt{\nu}}{K^{1/4}}=\frac{1}{4}\sqrt[4]{\frac{L(f(\x^1)-f^*)}{K\hat\sigma^2}}<1
\end{aligned}
\end{eqnarray}
for all $ k=1,2,\cdots,K$, leading to $\|\x^k\|_{\infty}<\frac{1}{\lambda}$.
\end{proof}

Finally, we summarize the key ideas behind the proof in the following two remarks.
\begin{remark}\label{remark1}
\textbf{Managing Weight Decay}. 

To manage the problematic term (b) introduced by weight decay, we first bound $|\lambda\x_i^k|$ in term (b) by $\frac{\sqrt{\nu}}{K^{1/4}}$ and relax terms (a) and (b) to terms (c) and (d), respectively. Then we can bound the second-order momentum in term (d) by Lemma \ref{lemma4} and absorb term (g) into term (e). By carefully choosing the bound of $|\lambda\x_i^k|$, we can control the magnitude of term (f). This sequence of operations enables us to manage the problematic term (b) and derive the bound (\ref{equ10}). 
\end{remark}
\begin{remark}\label{remark2}
\textbf{Eliminating the Dependence on $\varepsilon$}. 

Finally, we eliminate the dependence on $\varepsilon$ in term (h) by the tight bound of the second-order momentum in Lemma \ref{lemma4} and the careful configuration of $\varepsilon$. Although smaller value of $\varepsilon$ does not affect the convergence, this term cannot be eliminated any more and consequently slows the convergence rate by introducing explicit $\varepsilon$-dependence. On the other hand, while larger $\varepsilon$ does not impact the convergence rate, it makes AdamW closer to SGD.
\end{remark}
\subsection{Supporting Lemmas}\label{sec-supporting-lemmas}

This section provides several supporting lemmas to prove Theorem \ref{theorem2}. The second part of the following lemma bounds $\|\x^k\|_{\infty}$ for any $k$, which is a property unique to weight decay and can be found in \citep{xie-2024-adamw-icml}. We give the proof of the second part here only for the sake of completeness.
\begin{lemma}\label{lemma2}
Suppose $\m^0=0$, $\v^0=0$, $\theta\leq\tau\leq 1$, $0\leq\theta\leq \beta\leq\sqrt{\theta}<1$, $\lambda>0$, and $\eta\lambda<1$, then we have
\begin{eqnarray}
\begin{aligned}\notag
|\widetilde\m_i^k|^2\leq 8\v_i^k,\qquad\mbox{and}\qquad\|\x^{k+1}\|_{\infty}-\frac{3}{\lambda}\leq (1-\eta\lambda)^k\left(\|\x^1\|_{\infty}-\frac{3}{\lambda}\right).
\end{aligned}
\end{eqnarray}
\end{lemma}
\begin{proof}
For the first part, when $\theta=0$, we have $\beta=0$, $\m_i^k=\g_i^k=\widetilde\m_i^k$, $\v_i^k=|\g_i^k|^2$, and the inequality $|\widetilde\m_i^k|^2\leq 8\v_i^k$ trivially holds, so we only consider $\theta\neq 0$. From the recursions of $\m_i^k$, $\v_i^k$, and $\widetilde\m_i^k$, we have
\begin{eqnarray}
\begin{aligned}\notag
&\m_i^k=\theta^k\m_i^0+(1-\theta)\sum_{r=1}^{k}\theta^{k-r} \g_i^r=(1-\theta)\sum_{r=1}^{k}\theta^{k-r} \g_i^r,\\
&\v_i^k=\beta^k\v_i^0+(1-\beta)\sum_{r=1}^{k}\beta^{k-r} |\g_i^r|^2=(1-\beta)\sum_{r=1}^{k}\beta^{k-r} |\g_i^r|^2,\\
&\widetilde\m_i^k=\tau(1-\theta)\sum_{r=1}^{k}\theta^{k-r} \g_i^r+(1-\tau)\g_i^k=\tau(1-\theta)\sum_{r=1}^{k-1}\theta^{k-r} \g_i^r+(1-\tau\theta)\g_i^k.
\end{aligned}
\end{eqnarray}
Using H\"older's inequality, we have
\begin{eqnarray}
\begin{aligned}\notag
|\widetilde\m_i^k|^2=&\left(\tau(1-\theta)\sum_{r=1}^{k-1}\theta^{k-r} \g_i^r+(1-\tau\theta)\g_i^k\right)^2\\
\leq&\left(\sum_{r=1}^{k-1}\beta^{k-r} |\g_i^r|^2+|\g_i^k|^2\right)\left(\tau^2(1-\theta)^2\sum_{r=1}^{k-1}\left(\frac{\theta^2}{\beta}\right)^{k-r}+(1-\tau\theta)^2\right)\\
=&\frac{\v_i^k}{1-\beta}\left(\tau^2(1-\theta)^2\sum_{r=1}^{k-1}\left(\frac{\theta^2}{\beta}\right)^{k-r}+(1-\tau\theta)^2\right)\\
\overset{a}\leq& \frac{\v_i^k}{1-\beta}\left(\frac{(1-\theta)^2}{1-\frac{\theta^2}{\beta}}+(1-\theta^2)^2\right)\overset{b}\leq\v_i^k\left(\frac{(1-\theta)^2}{(1-\beta)^2}+\frac{(1-\theta^2)^2}{1-\beta}\right)\\
\overset{c}\leq&\v_i^k\left(\frac{(1-\sqrt{\theta})^2(1+\sqrt{\theta})^2}{(1-\sqrt{\theta})^2}+\frac{(1-\sqrt{\theta})(1+\sqrt{\theta})(1+\theta)(1-\theta^2)}{1-\sqrt{\theta}}\right)\\
=& \v_i^k\left((1+\sqrt{\theta})^2+(1+\sqrt{\theta})(1+\theta)(1-\theta^2)\right)\leq 8\v_i^k,
\end{aligned}
\end{eqnarray}
where we use $\theta\leq\tau\leq 1$ and $\frac{\theta^2}{\beta}<\frac{\theta}{\beta}\leq 1$ in $\overset{a}\leq$, $\theta\leq \beta$ in $\overset{b}\leq$, and $\beta\leq\sqrt{\theta}$ in $\overset{c}\leq$. For the second part, from the update of $\x^{k+1}$ with weight decay, we have
\begin{eqnarray}
\begin{aligned}\notag
\|\x^{k+1}\|_{\infty}-\frac{3}{\lambda}=& \left\|(1-\eta\lambda)\x^k-\frac{\eta}{\sqrt{\v^k+\varepsilon}}\odot\widetilde\m^k\right\|_{\infty}-\frac{3}{\lambda}\\
\leq& (1-\eta\lambda)\|\x^k\|_{\infty}+\left\|\frac{\eta}{\sqrt{\v^k+\varepsilon}}\odot\widetilde\m^k\right\|_{\infty}-\frac{3}{\lambda}\\
\overset{d}\leq& (1-\eta\lambda)\|\x^k\|_{\infty}+3\eta-\frac{3}{\lambda}\\
=& (1-\eta\lambda)\left(\|\x^k\|_{\infty}-\frac{3}{\lambda}\right)\leq (1-\eta\lambda)^k\left(\|\x^1\|_{\infty}-\frac{3}{\lambda}\right), \end{aligned}
\end{eqnarray}
where we use $|\widetilde\m_i^k|^2\leq 8\v_i^k$ in $\overset{d}\leq$.
\end{proof}

The following lemma establishes a recursion to bound the approximation error between the momentum and gradient.
\begin{lemma}\label{lemma3}
Suppose that Assumptions 1-3 hold. Let $0\leq\theta\leq\tau\leq 1$. Then for any $k\geq 2$, we have
\begin{eqnarray}
\begin{aligned}\label{l4-equ1}
&\E_k\left[\left\|\m^k-\nabla f(\x^k)\right\|^2|\F_{k-1}\right]\\
\leq&\theta\left\|\m^{k-1}-\nabla f(\x^{k-1})\right\|^2 + \frac{L^2\eta^2}{\sqrt{\varepsilon}(1-\theta)}\sum_{i=1}^d\frac{\left|\widetilde\m_i^{k-1}+\lambda\x_i^{k-1}\sqrt{\v_i^{k-1}+\varepsilon}\right|^2}{\sqrt{\v_i^{k-1}+\varepsilon}}+(1-\theta)^2\sigma^2
\end{aligned}
\end{eqnarray}
and
\begin{eqnarray}
\begin{aligned}\label{l4-equ2}
&\E_k\left[\left\|\widetilde\m^k-\nabla f(\x^k)\right\|^2|\F_{k-1}\right]\\
\leq& \theta\left\|\m^{k-1}-\nabla f(\x^{k-1})\right\|^2 + \frac{L^2\eta^2}{\sqrt{\varepsilon}(1-\theta)}\sum_{i=1}^d\frac{\left|\widetilde\m_i^{k-1}+\lambda\x_i^{k-1}\sqrt{\v_i^{k-1}+\varepsilon}\right|^2}{\sqrt{\v_i^{k-1}+\varepsilon}}+2(1-\theta)\sigma^2.
\end{aligned}
\end{eqnarray}
\end{lemma}
\begin{proof}
Denoting $\xi^k=\g^k-\nabla f(\x^k)$, we have $\E_k[\xi^k|\F_{k-1}]=0$ and $\E_k[\|\xi^k\|^2|\F_{k-1}]\leq\sigma^2$ from Assumptions 2-3. From the updates of $\m^k$ and $\widetilde\m^k$ in Algorithm \ref{nadamw}, we have
\begin{eqnarray}
\begin{aligned}\notag
\m^k-\nabla f(\x^k)=&\theta\m^{k-1}+(1-\theta)\g^k - \nabla f(\x^k)\\
=&\theta\left(\m^{k-1}-\nabla f(\x^{k-1})\right)+(1-\theta)\left(\nabla f(\x^k)+\xi^k\right) - \nabla f(\x^k) + \theta\nabla f(\x^{k-1})\\
=&\theta\left(\m^{k-1}-\nabla f(\x^{k-1})\right)+(1-\theta)\xi^k - \theta\left(\nabla f(\x^k)-\nabla f(\x^{k-1})\right)
\end{aligned}
\end{eqnarray}
and
\begin{eqnarray}
\begin{aligned}\notag
\widetilde\m^k-\nabla f(\x^k)=&\tau\m^k+(1-\tau)\g^k - \nabla f(\x^k)\\
=&\tau\theta\m^{k-1}+\tau(1-\theta)\g^k+(1-\tau)\g^k - \nabla f(\x^k)\\
=&\tau\theta\left(\m^{k-1}\hspace*{-0.06cm}-\hspace*{-0.06cm}\nabla f(\x^{k-1})\right)\hspace*{-0.06cm}+\hspace*{-0.06cm}(1-\tau\theta)\left(\nabla f(\x^k)+\xi^k\right) - \nabla f(\x^k) + \tau\theta\nabla f(\x^{k-1})\\
=&\tau\theta\left(\m^{k-1}-\nabla f(\x^{k-1})\right)+(1-\tau\theta)\xi^k - \tau\theta\left(\nabla f(\x^k)-\nabla f(\x^{k-1})\right).
\end{aligned}
\end{eqnarray}
So we have
\begin{eqnarray}
\begin{aligned}\notag
&\E_k\left[\left\|\m^k-\nabla f(\x^k)\right\|^2|\F_{k-1}\right]\\
=&\left\|\theta\left(\m^{k-1}-\nabla f(\x^{k-1})\right) - \theta\left(\nabla f(\x^k)-\nabla f(\x^{k-1})\right)\right\|^2+(1-\theta)^2\E_k\left[\|\xi^k\|^2|\F_{k-1}\right]\\
\leq&\theta^2\left(\hspace*{-0.06cm} \left(1+\frac{1-\theta}{\theta}\right)\hspace*{-0.06cm}\left\|\m^{k-1}\hspace*{-0.06cm}-\hspace*{-0.06cm}\nabla f(\x^{k-1})\right\|^2 \hspace*{-0.06cm}+\hspace*{-0.06cm} \left(\hspace*{-0.06cm}1\hspace*{-0.06cm}+\hspace*{-0.06cm}\frac{\theta}{1-\theta}\right)\hspace*{-0.06cm}\left\|\nabla f(\x^k)\hspace*{-0.06cm}-\hspace*{-0.06cm}\nabla f(\x^{k-1})\right\|^2 \right)\hspace*{-0.06cm}+\hspace*{-0.06cm}(1\hspace*{-0.06cm}-\hspace*{-0.06cm}\theta)^2\sigma^2\\
\overset{a}\leq&\theta\left\|\m^{k-1}-\nabla f(\x^{k-1})\right\|^2 + \frac{L^2}{1-\theta}\left\|\x^k-\x^{k-1}\right\|^2+(1-\theta)^2\sigma^2\\
=&\theta\left\|\m^{k-1}-\nabla f(\x^{k-1})\right\|^2 + \frac{L^2\eta^2}{1-\theta}\sum_{i=1}^d\frac{\left|\widetilde\m_i^{k-1}+\lambda\x_i^{k-1}\sqrt{\v_i^{k-1}+\varepsilon}\right|^2}{\v_i^{k-1}+\varepsilon}+(1-\theta)^2\sigma^2\\
\leq&\theta\left\|\m^{k-1}-\nabla f(\x^{k-1})\right\|^2 + \frac{L^2\eta^2}{\sqrt{\varepsilon}(1-\theta)}\sum_{i=1}^d\frac{\left|\widetilde\m_i^{k-1}+\lambda\x_i^{k-1}\sqrt{\v_i^{k-1}+\varepsilon}\right|^2}{\sqrt{\v_i^{k-1}+\varepsilon}}+(1-\theta)^2\sigma^2
\end{aligned}
\end{eqnarray}
and
\begin{eqnarray}
\begin{aligned}\notag
&\E_k\left[\left\|\widetilde\m^k-\nabla f(\x^k)\right\|^2|\F_{k-1}\right]\\
=&\left\|\tau\theta\left(\m^{k-1}-\nabla f(\x^{k-1})\right) - \tau\theta\left(\nabla f(\x^k)-\nabla f(\x^{k-1})\right)\right\|^2+(1-\tau\theta)^2\E_k\left[\|\xi^k\|^2|\F_{k-1}\right]\\
\leq&\tau^2\theta^2\hspace*{-0.05cm}\left(\hspace*{-0.05cm} \left(\hspace*{-0.05cm}1\hspace*{-0.08cm}+\hspace*{-0.08cm}\frac{1\hspace*{-0.08cm}-\hspace*{-0.08cm}\tau\theta}{\tau\theta}\right)\hspace*{-0.08cm}\left\|\m^{k-1}\hspace*{-0.08cm}-\hspace*{-0.08cm}\nabla f(\x^{k-1})\right\|^2 \hspace*{-0.08cm}+\hspace*{-0.08cm} \left(\hspace*{-0.08cm}1\hspace*{-0.08cm}+\hspace*{-0.08cm}\frac{\tau\theta}{1\hspace*{-0.08cm}-\hspace*{-0.08cm}\tau\theta}\right)\hspace*{-0.08cm}\left\|\nabla f(\x^k)\hspace*{-0.08cm}-\hspace*{-0.08cm}\nabla f(\x^{k-1})\right\|^2 \right)\hspace*{-0.08cm}+\hspace*{-0.08cm}(1\hspace*{-0.08cm}-\hspace*{-0.08cm}\tau\theta)^2\sigma^2\\
\overset{b}\leq&\tau\theta\left\|\m^{k-1}-\nabla f(\x^{k-1})\right\|^2 + \frac{L^2}{1-\tau\theta}\left\|\x^k-\x^{k-1}\right\|^2+(1-\tau\theta)^2\sigma^2\\
=&\tau\theta\left\|\m^{k-1}-\nabla f(\x^{k-1})\right\|^2 + \frac{L^2\eta^2}{1-\tau\theta}\sum_{i=1}^d\frac{\left|\widetilde\m_i^{k-1}+\lambda\x_i^{k-1}\sqrt{\v_i^{k-1}+\varepsilon}\right|^2}{\v_i^{k-1}+\varepsilon}+(1-\tau\theta)^2\sigma^2\\
\leq&\tau\theta\left\|\m^{k-1}-\nabla f(\x^{k-1})\right\|^2 + \frac{L^2\eta^2}{\sqrt{\varepsilon}(1-\tau\theta)}\sum_{i=1}^d\frac{\left|\widetilde\m_i^{k-1}+\lambda\x_i^{k-1}\sqrt{\v_i^{k-1}+\varepsilon}\right|^2}{\sqrt{\v_i^{k-1}+\varepsilon}}+(1-\tau\theta)^2\sigma^2\\
\leq&\theta\left\|\m^{k-1}-\nabla f(\x^{k-1})\right\|^2 + \frac{L^2\eta^2}{\sqrt{\varepsilon}(1-\theta)}\sum_{i=1}^d\frac{\left|\widetilde\m_i^{k-1}+\lambda\x_i^{k-1}\sqrt{\v_i^{k-1}+\varepsilon}\right|^2}{\sqrt{\v_i^{k-1}+\varepsilon}}+2(1-\theta)\sigma^2,
\end{aligned}
\end{eqnarray}
where we use $0\leq\theta\leq\tau\leq 1$ and $(1-\tau\theta)^2\leq 1-\tau\theta\leq 1-\theta^2=(1+\theta)(1-\theta)\leq 2(1-\theta)$ in the last inequality. Note that the inequalities $\overset{a}\leq$ and $\overset{b}\leq$ also hold for $\theta=0$, even though $\theta$ appears in the denominator in the preceding steps.
\end{proof}

The following lemma, already proved in \citep{lihuan-rmsprop-2024}, gives a tight upper bound for the second-order momentum in general adaptive-gradient-type methods. Our approach differs from the proof in \citep{lihuan-rmsprop-2024} in that they adopt the coordinate-wise bounded noise variance assumption, whereas we relax this assumption to bounded noise variance. Intuitively, when $\nabla_i f(\x^k)\approx 0$ and $|\nabla_i f(\x^k)-\g_i^k|^2\approx \sigma_i^2$ such that $\widetilde\v_i^k=\beta^k\widetilde\v_i^0+2(1-\beta)\sum_{t=1}^k\beta^{k-t}\left(\left|\nabla_i f(\x^t)\right|^2+\left|\nabla_i f(\x^t)-\g_i^t\right|^2\right)\approx 2\sigma_i^2$, we have $\sum_{k=1}^K\sum_{i=1}^d \sqrt{\widetilde\v_i^k+\varepsilon}\approx \sqrt{2}K\sum_{i=1}^d\sigma_i+Kd\sqrt{\varepsilon}\leq K\sqrt{2d}\sigma+Kd\sqrt{\varepsilon}$, leading to the first term on the right-hand side of (\ref{wide-v-bound}). 
\begin{lemma}\label{lemma4}
Suppose that Assumptions 2-3 hold. Let $0\leq\beta\leq 1$ and $\v^0=0$. Then for $\widetilde\v^k$ defined in (\ref{wide-v-def}), we have
\begin{eqnarray}
\begin{aligned}\label{wide-v-bound}
\sum_{k=1}^K\sum_{i=1}^d \E_{\F_k}\left[\sqrt{\widetilde\v_i^k+\varepsilon}\right]\leq \sqrt{2d}K\sigma+Kd\sqrt{\varepsilon}+4\sum_{t=1}^K\sum_{i=1}^d\E_{\F_t}\left[\frac{|\nabla_i f(\x^t)|^2}{\sqrt{\widetilde\v_i^t+\varepsilon}}\right].
\end{aligned}
\end{eqnarray}
\end{lemma}
\begin{proof}
Define $\n_i^k$ recursively as follows
\begin{eqnarray}
\begin{aligned}\label{n-def}
\n_i^k=\beta\n_i^{k-1}+2(1-\beta)|\g_i^k-\nabla_i f(\x^k)|^2\quad \mbox{with}\quad \n_i^0=0.
\end{aligned}
\end{eqnarray}
Recall the definition of $\widetilde\v_i^k$ in (\ref{wide-v-def}), we have
\begin{eqnarray}
\begin{aligned}\notag
\sqrt{\widetilde\v_i^k+\varepsilon}-\sqrt{\n_i^k+\varepsilon}=&\frac{\widetilde\v_i^k-\n_i^k}{\sqrt{\widetilde\v_i^k+\varepsilon}+\sqrt{\n_i^k+\varepsilon}}\\
=&\beta\frac{\widetilde\v_i^{k-1}-\n_i^{k-1}}{\sqrt{\widetilde\v_i^k+\varepsilon}+\sqrt{\n_i^k+\varepsilon}}+2(1-\beta)\frac{|\nabla_i f(\x^k)|^2}{\sqrt{\widetilde\v_i^k+\varepsilon}+\sqrt{\n_i^k+\varepsilon}}\\
\overset{a}\leq&\sqrt{\beta}\frac{\widetilde\v_i^{k-1}-\n_i^{k-1}}{\sqrt{\widetilde\v_i^{k-1}+\varepsilon}+\sqrt{\n_i^{k-1}+\varepsilon}}+2(1-\beta)\frac{|\nabla_i f(\x^k)|^2}{\sqrt{\widetilde\v_i^k+\varepsilon}}\\
=&\sqrt{\beta}\left(\sqrt{\widetilde\v_i^{k-1}+\varepsilon}-\sqrt{\n_i^{k-1}+\varepsilon}\right)+2(1-\beta)\frac{|\nabla_i f(\x^k)|^2}{\sqrt{\widetilde\v_i^k+\varepsilon}},
\end{aligned}
\end{eqnarray}
where we use $\widetilde\v_i^{k-1}\geq \n_i^{k-1}$, $\widetilde\v_i^k+\varepsilon\geq \beta\widetilde\v_i^{k-1}+\beta\varepsilon$, and $\n_i^k+\varepsilon\geq \beta\n_i^{k-1}+\beta\varepsilon$ in $\overset{a}\leq$. From this recursion and $\widetilde\v_i^0=\n_i^0=0$, we further have
\begin{eqnarray}
\begin{aligned}\notag
\sqrt{\widetilde\v_i^k+\varepsilon}-\sqrt{\n_i^k+\varepsilon}\leq 2(1-\beta)\sum_{t=1}^k\beta^{\frac{k-t}{2}} \frac{|\nabla_i f(\x^t)|^2}{\sqrt{\widetilde\v_i^t+\varepsilon}}.
\end{aligned}
\end{eqnarray}
Taking expectation with respect to $\F_k$, summing over $k=1,\cdots,K$ and $i=1,\cdots,d$, we have
\begin{eqnarray}
\begin{aligned}\label{equ11}
&\sum_{k=1}^K\sum_{i=1}^d\E_{\F_k}\left[\sqrt{\widetilde\v_i^k+\varepsilon}\right]\\
\leq& \sum_{k=1}^K\sum_{i=1}^d\E_{\F_k}\left[\sqrt{\n_i^k+\varepsilon}\right]+ 2(1-\beta)\sum_{k=1}^K\sum_{i=1}^d\sum_{t=1}^k\beta^{\frac{k-t}{2}} \E_{\F_t}\left[\frac{|\nabla_i f(\x^t)|^2}{\sqrt{\widetilde\v_i^t+\varepsilon}}\right]\\
\overset{b}\leq&\sqrt{Kd\sum_{k=1}^K\sum_{i=1}^d\E_{\F_k}\left[\n_i^k\right]}+Kd\sqrt{\varepsilon}+ 2(1-\beta)\sum_{k=1}^K\sum_{i=1}^d\sum_{t=1}^k\beta^{\frac{k-t}{2}} \E_{\F_t}\left[\frac{|\nabla_i f(\x^t)|^2}{\sqrt{\widetilde\v_i^t+\varepsilon}}\right]\\
\overset{c}=&\sqrt{2Kd(1-\beta)\sum_{k=1}^K\sum_{i=1}^d\sum_{t=1}^k\beta^{k-t}\E_{\F_t}\left[|\g_i^t-\nabla_i f(\x^t)|^2\right]}+Kd\sqrt{\varepsilon}\\
&+ 2(1-\beta)\sum_{k=1}^K\sum_{i=1}^d\sum_{t=1}^k\beta^{\frac{k-t}{2}} \E_{\F_t}\left[\frac{|\nabla_i f(\x^t)|^2}{\sqrt{\widetilde\v_i^t+\varepsilon}}\right],
\end{aligned}
\end{eqnarray}
where we use the concavity of $\sqrt{x}$ in $\overset{b}\leq$ and definition (\ref{n-def}) in $\overset{c}=$. When $\beta\neq 1$, we can check that
\begin{eqnarray}
\begin{aligned}\notag
\sum_{k=1}^K\sum_{i=1}^d\sum_{t=1}^k\beta^{k-t}\E_{\F_t}\left[|\g_i^t-\nabla_i f(\x^t)|^2\right]=&\sum_{t=1}^K\sum_{k=t}^K\beta^{k-t}\E_{\F_t}\left[\|\g^t-\nabla f(\x^t)\|^2\right]\\
\leq& \frac{1}{1-\beta}\sum_{t=1}^K\E_{\F_t}\left[\|\g^t-\nabla f(\x^t)\|^2\right]\leq \frac{K\sigma^2}{1-\beta},
\end{aligned}
\end{eqnarray}
and
\begin{eqnarray}
\begin{aligned}\notag
\sum_{k=1}^K\sum_{i=1}^d\sum_{t=1}^k\beta^{\frac{k-t}{2}} \E_{\F_t}\left[\frac{|\nabla_i f(\x^t)|^2}{\sqrt{\widetilde\v_i^t+\varepsilon}}\right]=&\sum_{i=1}^d\sum_{t=1}^K\sum_{k=t}^K\beta^{\frac{k-t}{2}} \E_{\F_t}\left[\frac{|\nabla_i f(\x^t)|^2}{\sqrt{\widetilde\v_i^t+\varepsilon}}\right]\\
\leq&\frac{1}{1-\sqrt{\beta}}\sum_{i=1}^d\sum_{t=1}^K\E_{\F_t}\left[\frac{|\nabla_i f(\x^t)|^2}{\sqrt{\widetilde\v_i^t+\varepsilon}}\right]\\
\leq&\frac{2}{1-\beta}\sum_{i=1}^d\sum_{t=1}^K\E_{\F_t}\left[\frac{|\nabla_i f(\x^t)|^2}{\sqrt{\widetilde\v_i^t+\varepsilon}}\right].
\end{aligned}
\end{eqnarray}
Plugging into (\ref{equ11}), we obtain the conclusion. When $\beta=1$, the conclusion holds trivially since the $1-\beta$ term in $\overset{c}=$.
\end{proof}

\subsection{Proof of Corollary \ref{corollary2}}\label{sec:corollary}
We give the complete description of Corollary \ref{corollary2} in the following corollary.
\begin{corollary}
Suppose that Assumptions 1-3 hold. Define $\hat\sigma^2=\max\left\{\sigma^2,\frac{L(f(\x^1)-f^*)}{K\gamma^2}\right\}$ with any constant $\gamma\in(0,1]$. Let $1-\theta=\sqrt{\frac{L(f(\x^1)-f^*)}{K\hat\sigma^2}}$, $\tau\in[\theta,1]$, $\beta\in[0,1]$, $\eta=\sqrt{\frac{f(\x^1)-f^*}{4KdL}}$, and $\varepsilon=\frac{\hat\sigma^2}{d}$. Then for NAdam, we have
\begin{eqnarray}
\begin{aligned}\notag
\frac{1}{K}\sum_{k=1}^K\E\left[\|\nabla f(\x^k)\|_1\right]\leq \frac{7\sqrt{d}}{K^{1/4}}\sqrt[4]{\hat\sigma^2L(f(\x^1)-f^*)} + 32\sqrt{\frac{dL(f(\x^1)-f^*)}{K}}.
\end{aligned}
\end{eqnarray}
In particular, when $\sigma^2\leq\frac{L(f(\x^1)-f^*)}{K\gamma^2}$, we have $1-\theta=\gamma$, $\tau\in[\theta,1]$, $\beta\in[0,1]$, $\eta=\sqrt{\frac{f(\x^1)-f^*}{4KdL}}$, $\varepsilon=\frac{L(f(\x^1)-f^*)}{dK\gamma^2}$, and accordingly
\begin{eqnarray}
\begin{aligned}\notag
\frac{1}{K}\sum_{k=1}^K\E\left[\|\nabla f(\x^k)\|_1\right]\leq 39\sqrt{\frac{dL(f(\x^1)-f^*)}{K\gamma}}.
\end{aligned}
\end{eqnarray}
\end{corollary}

\begin{proof}
When $\lambda=0$, the term $\frac{\eta\nu}{K^{1/2}}\sum_{i=1}^d\sqrt{\widetilde\v_i^k+\varepsilon}$ disappears in (\ref{equ4}) and (\ref{equ5}) in the proof of Theorem \ref{theorem2}, and (\ref{equ7}) becomes
\begin{eqnarray}
\begin{aligned}\notag
&\E_{\F_K}\left[f(\x^{K+1})-f^* + \frac{\eta\theta}{\sqrt{\varepsilon}(1-\theta)}\left\|\nabla f(\x^K)-\m^K\right\|^2 + \frac{\eta}{4}\sum_{i=1}^d\frac{\left|\widetilde\m_i^K+\lambda\x_i^K\sqrt{\v_i^K+\varepsilon}\right|^2}{\sqrt{\v_i^K+\varepsilon}} \right]\\
\leq& f(\x^1)-f^*-\frac{\eta}{2}\sum_{k=1}^K\sum_{i=1}^d\E_{\F_k}\left[\frac{\left|\nabla_i f(\x^k)\right|^2}{\sqrt{\widetilde\v_i^k+\varepsilon}}\right]+\frac{2\eta L(f(\x^1)-f^*)}{\sqrt{\varepsilon}(1-\theta)} + \frac{4K\eta(1-\theta)\sigma^2}{\sqrt{\varepsilon}},
\end{aligned}
\end{eqnarray}
where the term $2.5\eta\nu d\sqrt{K\varepsilon}$ disappears because we no longer need Lemma \ref{lemma4} to bound the term $\frac{\eta\nu}{K^{1/2}}\sum_{k=1}^K\sum_{i=1}^d\E_{\F_k}\left[\sqrt{\widetilde\v_i^k+\varepsilon}\right]$ any more. 

Similar to the proof of Theorem \ref{theorem2}, we have
\begin{eqnarray}
\begin{aligned}\notag
\sum_{k=1}^K\sum_{i=1}^d\E_{\F_k}\left[\frac{\left|\nabla_i f(\x^k)\right|^2}{\sqrt{\widetilde\v_i^k+\varepsilon}}\right]\leq& \frac{2(f(\x^1)-f^*)}{\eta}+\frac{4L(f(\x^1)-f^*)}{\sqrt{\varepsilon}(1-\theta)}+\frac{8K(1-\theta)\sigma^2}{\sqrt{\varepsilon}}\\
\leq&16\sqrt{\frac{K\hat\sigma^2L(f(\x^1)-f^*)}{\varepsilon}}.
\end{aligned}
\end{eqnarray}
Comparing with (\ref{equ9}), we see that the term $10\nu d\sqrt{K\varepsilon}$ disappears. Following the proof of Theorem \ref{theorem2}, we obtain the conclusion. Note that we do not need to bound $|\lambda\x_i^k|$, so Lemma \ref{lemma2} is no longer needed and Corollary \ref{corollary2} does not require $\theta\leq \beta\leq\sqrt{\theta}$ and $\|\x^1\|_{\infty}\leq\frac{\sqrt{\nu}}{4K^{1/4}\lambda}$ any more.
\end{proof}

\subsection{Proof of Lemma \ref{main-lemma}}

\begin{proof}
When $\x_i\sim \mathcal N(0,1)$, we have 
\begin{eqnarray}
\begin{aligned}\notag
&\E\left[|\x_i|\right]=\sqrt{\frac{2}{\pi}},\quad \E\left[\x_i^2\right]=1,\\
&\E\left[\|\x\|_1\right]=\sum_{i=1}^d \E\left[|\x_i|\right]=d\sqrt{\frac{2}{\pi}},\\
&\E\left[\|\x\|_2^2\right]=\sum_{i=1}^d \E\left[\x_i^2\right]=d,\\
&\E\left[\|\x\|_2\right]=\E\left[\sqrt{\|\x\|_2^2}\right]\overset{a}\leq \sqrt{\E\left[\|\x\|_2^2\right]}=\sqrt{d},\\
&\frac{\E\left[\|\x\|_1\right]}{\E\left[\|\x\|_2\right]}\geq \sqrt{\frac{2d}{\pi}}.
\end{aligned}
\end{eqnarray}
where we use the concavity of $\sqrt{x}$ in $\overset{a}\leq$.
\end{proof}

\section{Experimental Details}\label{sec5}
As presented in the preceding sections, we conduct several representative deep learning experiments to empirically support our claims, covering classic image classification and natural language processing tasks. For the vision tasks, we independently train ResNet-50 \citep{he2016deep} on CIFAR-100 \citep{krizhevsky2009learning} and ImageNet \citep{ILSVRC15} datasets. For the language task, we adopt the GPT-2 \citep{radford2019language} architecture and pretrain it on the OpenWebText \citep{Gokaslan2019OpenWeb} dataset. 
To promote maximal transparency and reproducibility, we release all code, configuration files, plotting scripts, and processed result data on GitHub. All experiments can be easily inspected and reproduced through the provided ready-to-run script. Link:
\begingroup
\urlstyle{rm}
\small\url{https://github.com/adonis-dym/Convergence-Rate-AdamW}
\endgroup.

Our experiments involve the computation of the full training loss $f(\x^k)$ as well as the full gradient $\nabla f(\x^k)$. However, in the typical stochastic training paradigm, one often updates the parameter $\x^k$ on the fly immediately after obtaining the stochastic gradient $\g^k$ from the backward pass. To get an accurate measurement and avoid interfering with the normal training process, we propose to split each epoch into two separate phases: \textit{training phase} and \textit{logging phase}. 
In the training phase, we traverse the dataset once with stochastic updates, where the model parameters are updated upon processing each mini-batch. In the logging phase, we conduct a second traversal over the training dataset while keeping the model parameters frozen. Since the loss function is typically defined to be the average over all training samples and the gradient computation is inherently linear, we accumulate the losses and stochastic gradients across mini-batches during this phase. This yields the exact values of the full training loss $f(\x^k)$ and full gradient $\nabla f(\x^k)$ at the current iteration. 

In the following, we detail each experimental setup:

\textit{i) ResNet-50 - CIFAR-100: } CIFAR-100 is a simple benchmark dataset that is widely used for efficient evaluation of deep learning models. It contains a training split of 50000 examples and a test split of 10000 examples, although we do not perform evaluation on the test set in this work. The AdamW optimizer is configured via the official \texttt{torch.optim.AdamW} interface, whereas NAdamW employs a commonly used implementation \citep{Kasimbeg2025AlgoPerfResults,kaiyue-compare-2025} aligned with Algorithm~\ref{nadamw} due to the absence of a unified reference, and is open-sourced for reproducibility. Except for the additional parameter $\tau$ introduced by NAdamW, all other settings, including hyperparameters, learning-rate schedule, and random seed, are kept identical across the two optimizers, with $\tau=\theta=0.9$ following the common practice. We initialize the learning rate to $3\times10^{-3}$, train the ResNet-50 model for 100 epochs, and apply a cosine learning-rate decay schedule during the whole training process. Setting the batch size to 128, each epoch consists of $\lfloor 50000 / 128 \rfloor + 1 = 391$ steps, where the additional step accounts for the final truncated batch which contains the remaining samples. The total number of steps is $K=391\times100=39100$. Without loss of generality, we compute the noise vector $\bsigma^k=\mathbf{g}^k - \nabla f(\mathbf{x}^k)$ using the stochastic gradient $\mathbf{g}^k$ obtained from the first batch at the logging phase. We leave the weight-decay parameter $\lambda$ as its default value 0.01, and complete the training task with a single NVIDIA A100 GPU.

\textit{ii) ResNet-50 - ImageNet: } 
To evaluate the scalability of our conclusions on larger-scale dataset, we conduct experiments on the ImageNet dataset using the same ResNet-50 architecture. ImageNet consists of approximately 1.28 million training images and 50,000 validation images across 1,000 classes, and has an official dataset split. We employ the training script from PyTorch Image Models (\texttt{timm}) \citep{rw2019timm}, making only the necessary modifications to suit our experimental setup. We adopt the same optimizer configuration as previously, but compute the noise vector using the last batch at the logging phase, as the \texttt{timm} script discards incomplete batches and ensures uniform batch sizes. We follow the standard ImageNet training protocol for ResNet-50, which consists of 90 epochs as commonly adopted in the literature and official implementations \citep{he2016deep, rw2019timm}. The first 10 epochs are used for learning-rate linear warmup from 0 to $3\times10^{-3}$, followed by cosine decay over the remaining 80 epochs. We apply standard data augmentation techniques including RandAugment, Mixup (0.1), and CutMix (1.0). With a batch size of 4096, each epoch consists of 312 mini-batches and the total number of steps is $K=28080$. We set $\lambda=0.1$ and complete the training task using eight NVIDIA A100 GPUs.

\textit{iii) GPT-2 - OpenWebText: } 
To assess the generality of our conclusions across different modalities, we further evaluate on a language-modeling task using GPT-2. We pretrain this model on the OpenWebText dataset under the NVIDIA Megatron-LM codebase \citep{shoeybi2019megatron}, which is a widely adopted framework for large-scale language model training. Unlike the previous settings, where computing the full training loss and gradient over the entire dataset is tractable, the OpenWebText dataset is substantially larger, containing approximately 9 billion tokens. Consequently, an entire pass through the dataset to get the full training loss $f(\x^k)$ and gradient $\nabla f(\x^k)$ is computationally infeasible. Instead, we approximate these quantities by accumulating their values over 100 consecutive mini-batches at the logging phase. We follow the Megatron-LM official GPT-2 training configuration with minimal modifications to suit our experimental needs. We train a GPT-2 Small model with approximately 125M parameters, setting the learning rate to $3\times10^{-3}$ and weight-decay parameter to 0.05. Following the de facto standard in large-scale language model training, we use $(\theta, \beta) = (0.9, 0.95)$ instead of the conventional $(0.9, 0.999)$ setting. We use the fused AdamW implementation from the NVIDIA Apex package, which is the default in Megatron-LM, and employ the same open-sourced version for NAdamW. Training runs for 50,000 iterations, where the learning rate is linearly warmed up for the first 2,000 iterations and then decayed following a cosine schedule. We set the global batch size to 640 and train the model for $K=50000$ steps, and complete training using eight NVIDIA A100 GPUs. Notably, the NAdamW optimizer attains a lower training loss than AdamW after approximately 20k steps, while exhibiting slightly higher loss in the early stage. We attribute this behavior to the well-documented gradient spike phenomenon in LLM training \citep{molybog2023theory,takase2023spike,wang2025adagc}, as NAdamW places greater emphasis on the current-step gradient and may therefore stabilize more slowly.

\section*{Conclusion}
This paper studies the popular AdamW optimizer in large language model training, along with its simplified predecessor Adam and its extension NAdamW. We establish the convergence rate $\frac{1}{K}\sum_{k=1}^K\E\left[\|\nabla f(\x^k)\|_1\right]\leq \bO(\frac{\sqrt{d}C}{K^{1/4}})$ for AdamW/Adam/NAdamW measured by the $\ell_1$ norm. It can be regarded as analogous to the optimal $\frac{1}{K}\sum_{k=1}^K\E\left[\|\nabla f(\x^k)\|_2\right]\leq \bO(\frac{C}{K^{1/4}})$ rate of SGD in the ideal case of $\|\nabla f(\x)\|_1=\varTheta(\sqrt{d})\|\nabla f(\x)\|_2$, which is verified on real-world deep learning tasks. 

An important direction for future research would be to investigate the optimal convergence rate using weak $\varepsilon$-dependent analysis (for example, $\log\frac{1}{\varepsilon}$) for AdamW/ Adam/NAdamW, or to prove that the theoretical setting $\varepsilon=\frac{\sigma^2}{d}$ is necessary for the optimal convergence rate. On the other hand, it is currently unclear whether our upper bound on the weight-decay parameter $\lambda$ is tight. Investigating how to prove the optimal convergence rate under a looser upper bound on $\lambda$ would be meaningful. 


\bibliography{AdamW}
\bibliographystyle{icml2020}

\end{document}